\documentclass[nohyperref]{article}

\usepackage{microtype}
\usepackage{graphicx}
\usepackage{subfigure}
\usepackage{booktabs} 

\usepackage{hyperref}

 \usepackage[accepted]{icml2022}

\usepackage{amsmath}
\usepackage{amssymb}
\usepackage{mathtools}
\usepackage{amsthm}

\usepackage{booktabs} 
\usepackage{diagbox}
\usepackage{bm}
\usepackage{bbding}

\usepackage[capitalize,noabbrev]{cleveref}

\theoremstyle{plain}
\newtheorem{theorem}{Theorem}[section]

\theoremstyle{definition}

\theoremstyle{remark}

\usepackage[textsize=tiny]{todonotes}

\icmltitlerunning{PDO-s3DCNNs: Partial Differential Operator Based Steerable 3D CNNs}

\begin{document}

\twocolumn[
\icmltitle{PDO-s3DCNNs: Partial Differential Operator Based Steerable 3D CNNs}



\icmlsetsymbol{equal}{*}

\begin{icmlauthorlist}
\icmlauthor{Zhengyang Shen}{math,byte}
\icmlauthor{Tao Hong}{math}
\icmlauthor{Qi She}{byte}
\icmlauthor{Jinwen Ma}{math}
\icmlauthor{Zhouchen Lin}{key,ai,pazhou}

\end{icmlauthorlist}
\icmlaffiliation{math}{School of Mathematical Sciences, Peking University, Beijing, China }
\icmlaffiliation{byte}{Bytedance AI Lab, Haidian District, Beijing, China}
\icmlaffiliation{key}{Key Lab. of Machine Perception (MoE), School of Artificial Intelligence, Peking University, Beijing, China}
\icmlaffiliation{ai}{Institute for Artificial Intelligence, Peking University, Beijing, China}
\icmlaffiliation{pazhou}{Pazhou Lab, Guangzhou, China}

\icmlcorrespondingauthor{Jinwen Ma}{jwma@math.pku.edu.cn}
\icmlcorrespondingauthor{Zhouchen Lin}{zlin@pku.edu.cn}

\icmlkeywords{Machine Learning, ICML}

\vskip 0.3in
]



\printAffiliationsAndNotice{}  

\begin{abstract}
Steerable models can provide very general and flexible equivariance by formulating equivariance requirements in the language of representation theory and feature fields, which has been recognized to be effective for many vision tasks. However, deriving steerable models for 3D rotations is much more difficult than that in the 2D case, due to more complicated mathematics of 3D rotations. In this work, we employ partial differential operators (PDOs) to model 3D filters, and derive general steerable 3D CNNs, which are called PDO-s3DCNNs. We prove that the equivariant filters are subject to linear constraints, which can be solved efficiently under various conditions. As far as we know, PDO-s3DCNNs are the most general steerable CNNs for 3D rotations, in the sense that they cover all common subgroups of $SO(3)$ and their representations, while existing methods can only be applied to specific groups and representations. Extensive experiments show that our models can preserve equivariance well in the discrete domain, and outperform previous works on SHREC'17 retrieval and ISBI 2012 segmentation tasks with a low network complexity.
\footnote{Source code is available at https://github.com/shenzy08/PDO-s3DCNN.}
\end{abstract}

\section{Introduction \label{introduction}}
In the past few years, convolutional neural networks (CNNs) have dominated the computer vision field on various tasks. Compared with fully-connected (FC) neural networks, one main advantage of CNNs is that they are inherently translation equivariant: shifting an input and then feeding it through a CNN is the same as feeding the original input and then shifting the resulted feature maps. However, conventional CNNs are not naturally equivariant to other transformations, such as rotations.

Consequently, many works focus on incorporating more equivariance into CNNs. Firstly adopted in 2D images, \citet{cohen2016group} proposed group equivariant CNNs (G-CNNs) by rotating filters, which is equivariant over the cyclic group $C_4$ or the dihedral group $D_4$. Then some methods \citep{hoogeboom2018hexaconv,weiler2018learning} are successively proposed to exploit the larger groups. However, the equivariance achieved by these methods is very inflexible, because the feature maps inherently can only transform in one single pattern (which can be characterized by regular representation in the language of representation theory) as inputs rotate. But in many vision tasks, it would be better if the transformation behavior of feature maps could be flexibly defined as needed. For example, if the vectors in feature maps are used to predict the object orientations, they should also rotate as the inputs rotate.

 In order to utilize more flexible and general equivariance, \citet{cohen2017steerable} proposed a theoretical framework called steerable CNNs to describe equivariant models in the language of representation theory and feature fields. To be specific, the feature spaces are defined as the spaces of feature fields, characterized by a group representation which determines (``steers'') their transformation behavior under the transformations of the input. After the feature fields or representations are specified, equivariant filters are obtained by directly solving the equivariance constraints. Steerable CNNs include G-CNNs as special cases when employing regular features. \citet{weiler2019general} further extended the theoretical framework of steerable CNNs from $C_4$ to the Euclidean group $E(2)$ and its subgroups, proposing E2CNNs. E2CNNs are the most general steerable CNNs for 2D rotations so far, as they can deal with all common subgroups of $SO(2)$ and their representations. However, E2CNN \emph{cannot} be extended to the 3D case trivially, since the mathematics of rotation equivariance in 3D is much more complicated than in 2D because 3D rotations do \emph{not} commute \cite{thomas2018tensor}. 

Actually, equivariance is more important for the 3D case because 3D rotations are inevitable -- even though a 3D object is up-right, rotations may still occur around the vertical axis, not to mention that some 3D data even have no inherent orientations, such as molecular data \citep{anderson2019cormorant}. In addition, there are relatively few works on 3D CNNs, due to their computation and storage complexity. Thus exploiting equivariance to reduce both computation and storage is critical for applications of 3D CNNs.
N-body networks \citep{kondor2018n}, Tensor Field Networks (TFNs) \citep{thomas2018tensor} and SE3CNNs \citep{weiler20183d} succeeded in solving the equivariance constraints when defining the feature spaces via irreducible representations (irreps) of $SO(3)$. However, the constraints for discrete subgroups are still not solved.
CubeNets \citep{worrall2018cubenet} exploit equivariance over discrete subgroups by applying G-CNNs to 3D data, i.e., rotating 3D filters. However, viewed as steerable CNNs, CubeNets can only accommodate regular representations of the cubic group $\mathcal{O} $ and its subgroups, noting that there are only cubic rotational symmetries on regular 3D grids. They are not applicable to larger groups (e.g., the icosahedral group $\mathcal{I}$ and $SO(3)$) and other feature fields (e.g., quotient features). 
In all, existing methods can only deal with specific groups and representations, while the theory, analogous to E2CNN in the 2D case, that can accommodate all common subgroups of $SO(3)$ and their representations is still missing.

\begin{table*}[t]
	\small
	\caption{The comparison between PDO-s3DCNNs and other 3D steerable models. Ours can accommodate all common subgroups of $SO(3)$ and feature fields, while others can only address specific groups and feature fields.} \smallskip
	\centering
	\begin{tabular}{lccccccc}
		\toprule
		& \multicolumn{3}{c}{Group $\mathcal{G}$} & \multicolumn{3}{c}{Feature field} \\
		\cmidrule(r){2-4} \cmidrule(r){5-7}
		Method  &  $\mathcal{G}\leq \mathcal{O}$ & $\mathcal{I}$ & $SO(3)$ & Regular & Quotient & Irreducible & Data type\\
		\midrule
		N-body networks \citep{kondor2018n}  &&  & \Checkmark  &  &  & \Checkmark  & graphs\\
		TFN \citep{thomas2018tensor} & &  & \Checkmark  &  & & \Checkmark  & point clouds \\
		CubeNets \citep{worrall2018cubenet}  & \Checkmark &  &  & \Checkmark& & & voxels  \\
		SE3CNNs \citep{weiler20183d} & & & \Checkmark  & &  & \Checkmark  & voxels \\
		SE(3)-Transformers \citep{fuchs2020se}  &&  & \Checkmark  &  &  & \Checkmark  & point clouds/graphs\\
		\hline
		PDO-s3DCNNs (Ours) &\Checkmark &\Checkmark&\Checkmark&\Checkmark&\Checkmark&\Checkmark & voxels\\
		\bottomrule
	\end{tabular}
	\label{summary}
\end{table*}


Besides theoretical significance, we would like to emphasize that although the equivariance over a continuous group covers the equivariance over its discrete subgroups, empirical results \citep{weiler20183d} show that the models equivariant over discrete subgroups perform even better\footnote{Similar results also occur in our experiments, and we argue that it is due to the amenable normalizations and nonlinearities, as discussed in Sections \ref{cdlt} and \ref{shrec17}.}. 
This phenomenon explains why we still need to investigate the equivariance over discrete subgroups in practice even though some works have achieved $SO(3)$-equivariance.

In this work, we are devoted to proposing general steerable CNNs for 3D rotations and filling in the deficiency of existing works, and the comparison between our theory and other 3D steerable CNNs is summarized in Table \ref{summary}. Specifically,  \citet{shen2020pdo,shen2021pdo} showed that partial differential operators (PDOs) are very effective for designing equivariant CNNs, which are convenient for mathematical derivation in the continuous domain. Following them, we model 3D filters using PDOs, hence our models are called PDO-s3DCNNs. We emphasize that our work is \emph{not} a trivial extension of PDO-eConv \cite{shen2020pdo} to the 3D case, because PDO-eConv is limited to the \emph{regular features} of \emph{discrete subgroups}, and cannot address other feature fields and continuous groups whereas ours can. \citet{jenner2021steerable} successfully proposed steerabble PDOs in the 2D case. However, their theory cannot deal with arbitrary 3D rotation groups and their representations, such as discrete subgroups, due to the complexity of 3D rotations. In our theory, we prove that given any specific rotation group and feature field, the equivariance requirement for a PDO-based filter can be deduced to a linear constraint for its coefficients. Then we illustrate how to solve this constraint efficiently under various conditions. In the implementation, it is easy to discretize our derived PDO-based equivariant filters on volumetric inputs using Finite Difference (FD) or Gaussian discretization. Experiments verify that our methods can preserve equivariance well in the discrete domain, and outperform their counterparts, including SE3CNN \cite{weiler20183d} and CubeNet \cite{worrall2018cubenet}, on SHREC'17 retrieval and ISBI 2012 segmentation tasks, respectively, with a low network complexity.

In summary, our contributions are as follows:
\begin{itemize}
	\item We propose general steerable 3D CNNs for rotations, which cover both the continuous group $SO(3)$ and its discrete subgroups. In contrast, existing works cover either $SO(3)$ or its discrete subgroups, but not both.
	\item It is the first time that PDOs are utilized to design equivariant 3D CNNs, and also the first time that the quotient feature is investigated in the 3D case. Note that the quotient feature is a very flexible feature field, including the regular feature as a special case.
	\item We employ more discretization schemes than that in PDO-eConvs for implementation, and show that Gaussian discretization can preserve the equivariance for large groups in the discrete domain much better than FD.
	\item In experiments, our models perform significantly better than their counterparts SE3CNN on SHREC'17 retrieval task, and CubeNet on ISBI 2012 segmentation task, respectively.
\end{itemize}

\section{Related Work}

\subsection{Equivariant 2D CNNs}
\citet{lenc2015understanding} observed that CNNs spontaneously learn representations equivariant to some transformations, indicating that equivariance is a good inductive bias for CNNs. \citet{cohen2016group} succeeded in incorporating rotation equivariance into neural networks by group equivariant correlation, and proposed G-CNNs. However, this method can only deal with the 4-fold rotational symmetry. Thus some follow-up works \citep{zhou2017oriented,marcos2017rotation,hoogeboom2018hexaconv,weiler2018learning,shen2020pdo} focused on exploiting larger discrete groups, as more symmetries make models more data-efficient \citep{he2021efficient}. Also, there are some works \citep{worrall2017harmonic,esteves2018polar,finzi2020generalizing} managing to achieve rotation equivariance over the continuous group $SO(2)$. 

The above methods mostly achieve equivariance by using group equivariant correlation. To exploit more general equivariance, \citet{cohen2017steerable} proposed steerable CNNs, where feature spaces are characterized as feature fields. Then the equivariant filters are obtained by solving equivariance constraints directly. E2CNNs \citep{weiler2019general} are the most general steerable CNNs for 2D rotations, because they involve both the continuous group $SO(2)$ and its discrete subgroups into a unified framework. It motivates us to develop equally general steerable 3D CNNs for rotations, as such theory is still missing in the 3D case. 

\subsection{Equivariant 3D Models}
Some methods have achieved equivariance over 3D rotations by solving equivariant constraints, and they are successfully applied to graphs \citep{kondor2018n}, point clouds \citep{thomas2018tensor}, volumetric data \citep{weiler20183d} and mesh data \citep{he2021gauge}, respectively. Based on TFN \citep{thomas2018tensor}, \citet{fuchs2020se} proposed SE(3)-Transformer by designing equivariant self-attention. However, these methods can only deal with the continuous group $SO(3)$. \citet{lang2020wigner} provided a general characterization of equivariant filters for any compact group. Even so, the case for discrete subgroups of $SO(3)$ is still not solved for practical use due to the complexity of 3D rotations \cite{reisert2009spherical}. There are also some works exploiting the equivariance over discrete subgroups. \citet{worrall2018cubenet} proposed CubeNets by applying the idea of G-CNNs, which is equivariant over the cubic group and its subgroups. A similar idea is applied to the medical image analysis \citep{winkels2019pulmonary}. However, these methods cannot exploit larger groups because there are no more rotational symmetries on regular 3D grids, while our method can break this limit. \citet{chen2021equivariant} achieved approximate $SO(3)$-equivariance by discretizing the continuous group $SO(3)$ and sampling over point clouds, and their equivariance is very weak.

Besides, some methods achieve rotation equivariance on other data types. Specifically, some projected 3D objects to the sphere and designed rotation equivariant spherical CNNs \citep{cohen2018spherical,esteves2018learning,kondor2018clebsch,esteves2020spin}. \citet{esteves2019equivariant} applied the idea of G-CNNs to process multiple views of 3D inputs. Unfortunately, these methods lose translation equivariance. By contrast, our method is translation equivariant because PDOs are naturally translation equivariant.

\section{PDO-s3DCNNs \label{section3}}
\subsection{Prior Knowledge}
\textbf{Equivariance in 3D}\quad Equivariance indicates that the outputs of a mapping transform in a predictable way with the transformations of the inputs. To be specific, let $\Psi$ be a filter, which could be represented by a layer of neural network from the input feature space to the output feature space, and $\mathcal{G}$ is a transformation group. $\Psi$ is called group equivariant over $\mathcal{G}$ if it satisfies that $\forall g\in \mathcal{G}$,
\begin{equation}
	\pi'(g)\left[\Psi\left[f\right]\right]=\Psi\left[\pi(g)[f]\right],\label{mequivariance}
\end{equation}
where $f$ can be represented as a stack of feature maps $f_k$ (for $k=1,2,\cdots,K$) in deep learning. For ease of derivation, we further suppose each $f_k$ to be a smooth function over $\mathbb{R}^3$, then we have that $f\in C^{\infty}(\mathbb{R}^3,\mathbb{R}^K)$. $\pi(g)$ and $\pi'(g)$ are called \emph{group actions}, where $\pi(g)$ describes how the transformation $g$ acts on inputs, and $\pi'(g)$ allows us to ``steer'' the resulted feature $\Psi[f]$ without referring to the input $f$. In addition, since we hope that two transformations $g,h\in \mathcal{G}$ acting on the feature maps successively is equivalent to the composition of transformations $gh\in \mathcal{G}$ acting on the feature maps directly, we require that $\pi(g)\pi(h)=\pi(gh)$. The same is the case with $\pi'(g)$. 

In this work, we focus on the equivariance over rotations, so $\mathcal{G}$ is taken as a rotation group in the sequel. Since we employ PDOs to design equivariant 3D CNNs, our models are naturally translation equivariant.

\begin{figure}[t]
	\centering
	\includegraphics[scale=0.35]{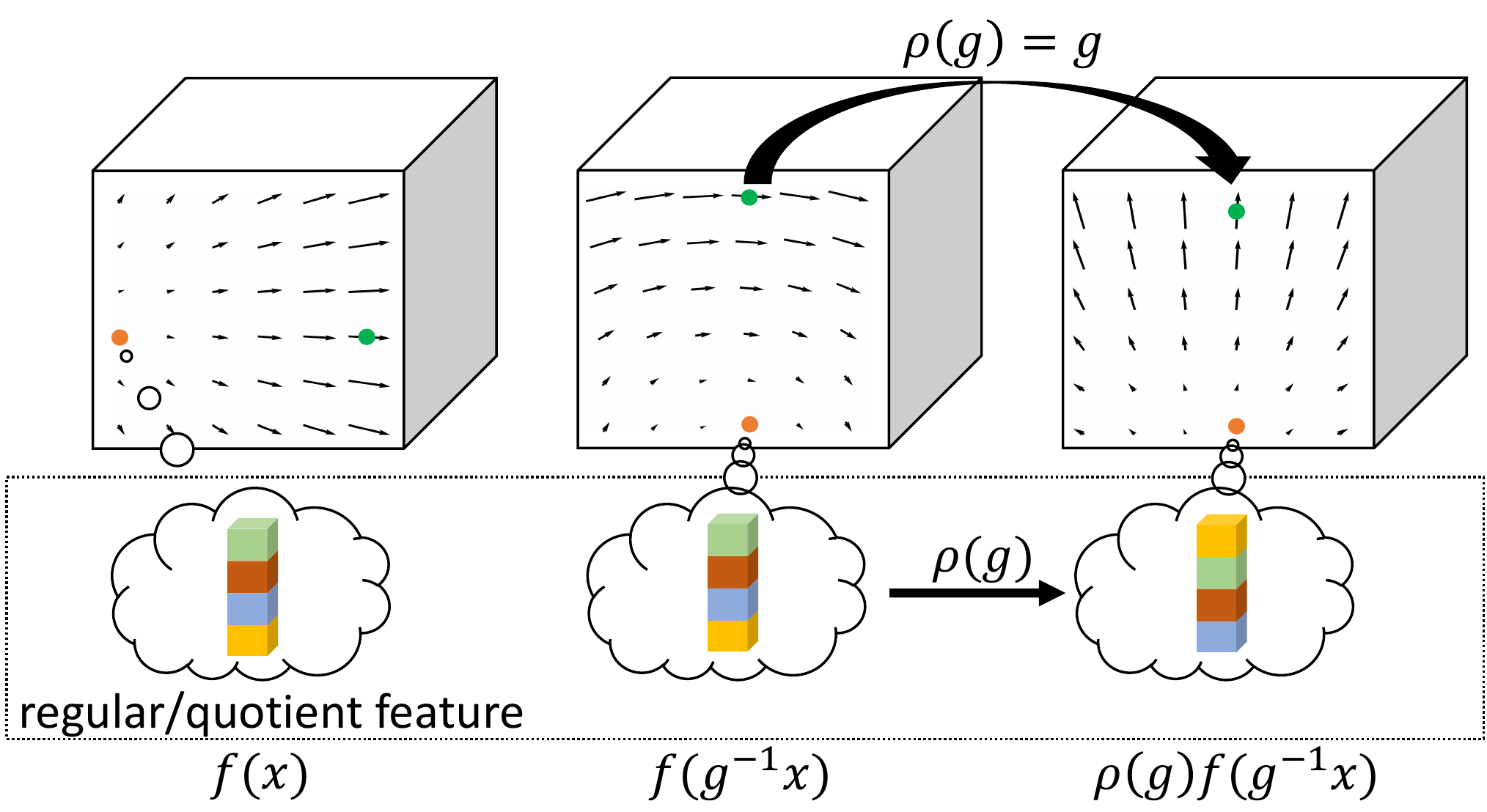}
	\caption{The schema of a rotation $g$ acting on a $\rho$-field. If $\rho(g)=(1)$, each feature vector is moved to its new position; if $\rho(g)=g$, each feature vector additionally rotates according to $g$; if $\rho(g)$ is a regular or quotient representation, the values of each feature vector are permutated. }
	\label{field}
\end{figure}

\noindent\textbf{Feature Fields and Group Representations}\quad Now we examine rotation group actions $\pi(g)$. The corresponding feature $f$ which is transformed according to $g$ is called a \emph{feature field}\footnote{Without ambiguity, we may omit ``field" sometimes for ease of presentation.}. For example, for input 3D data, the action can be naturally formulated as 
\begin{equation}
	[\pi(g)f](x) = f(g^{-1}x),
	\label{scalar}
\end{equation}
and the corresponding feature is called a \emph{scalar feature field}. If the output is also taken as a scalar feature, the filter $\Psi$ should be isotropic and restricts the capacity of neural networks \citep{cohen2019gauge}.

In order to address this problem, we consider the following much more general feature field:
\begin{equation}
	\left[\pi(g)f\right](x) = \rho(g)f\left(g^{-1}x\right),
	\label{rho}
\end{equation}
where $\rho(g)$ is a $K\times K$ matrix and determines how a rotation $g$ acts on the feature, and the corresponding feature field is called a $\rho$-field. $\pi(g)\pi(h)=\pi(gh)$ indicates that $\rho(g)\rho(h)=\rho(gh)$, so $\rho(g)$ is essentially a \emph{group representation}\footnote{This paper involves some terminologies in the group theory, and readers may refer to Appendix \ref{appa} and \citep{artin2011algebra,serre1977linear} for more details if interested.} of the group $\mathcal{G}$. The simplest examples are the trivial representation $\rho(g)=(1)$, which is a one-dimensional identical matrix and exactly corresponds to the scalar feature field. When $\rho(g)=g$, where $g$ can be parameterized as a $3\times 3$ rotation matrix, each feature vector $f(x)\in \mathbb{R}^3$ should additionally rotate according to $g$. In this way, if a feature vector is used to predict the object orientation, it can transform properly as the inputs rotate. The schema of a rotation $g$ acting on a  $\rho$-field is shown in Figure \ref{field}.

As will be shown in Section \ref{section4}, some common feature fields, e.g., regular, quotient, and irreducible features, which have been investigated in E2CNNs \citep{weiler2019general}, are re-defined in the 3D case and can be easily addressed by our theory. By contrast, previous works on steerable 3D CNNs can only deal with one of these features. For example, CubeNets \citep{worrall2018cubenet} and SE3CNNs \citep{weiler20183d} can only accommodate regular and irreducible features, respectively. Also, it is the first time that quotient features are investigated in the 3D case, which increases the flexibility of equivariant CNNs and allows us to decouple the computational cost from the group size.

\subsection{PDO-based Equivariant 3D Filters \label{kernels}}
\citet{shen2020pdo,shen2021pdo} have shown that PDOs are very effective for deriving equivariant models, as they are convenient for mathematical derivation. Following them, we employ a combination of PDOs to define a 3D filter on the input function $f\in C^{\infty}(\mathbb{R}^3,\mathbb{R}^{K})$:
\begin{align}
	\Psi[f] =& \left(A_0+\,A_1\partial_{x_1}+A_2\partial_{x_2}+A_3\partial_{x_3}+A_{11}\partial_{x_1^2}
	\right.\notag\\
	&\left.+A_{12}\partial_{x_1x_2}+A_{13}\partial_{x_1x_3}+A_{22}\partial_{x_2^2}+A_{23}\partial_{x_2x_3}\right.\notag\\
	&\left.+A_{33}\partial_{x_3^2}\right)[f],
	\label{mPsi}
\end{align}
where the coefficients $A_i\in \mathbb{R}^{K’\times K}$, and $K'$ is the number of channels of the output feature. We employ the PDOs up to the second-order, and higher-order PDOs can be obtained by stacking multiple $\Psi$'s. Then the requirement that $\Psi$ is equivariant can be deduced to a linear constraint for the coefficients. Supposing that the input and output features are $\rho$- and $\rho'$-field, respectively, we have the following theorem\footnote{All the detailed proofs can be found in Appendices \ref{appb} and \ref{appd}.}.

\begin{theorem}
	$\Psi$ is equivariant over $\mathcal{G}$, if and only if its coefficients satisfy the following linear constraints: $\forall g\in \mathcal{G}$,
	\begin{equation}
		\left\{
		\begin{array}{l}
			\rho'(g)B_0 = B_0\rho(g),\\
			\rho'(g)B_1 =B_1\left(g\otimes \rho(g)\right),\\
			\rho'(g)B_2 =B_2\left(P\left(g\otimes g\right)P^{\dag}\otimes\rho(g)\right),
		\end{array}
		\right.
		\label{mbase}
	\end{equation}
	i.e.,
	\begin{scriptsize}
		\begin{equation}
			\left\{
			\begin{array}{l}
				\left(I_K \otimes\rho'(g)-\rho(g)^T\otimes I_{K'}\right)vec(B_0) =0,\\
				\left(I_{3K} \otimes\rho'(g)-\left(g\otimes \rho(g)\right)^T\otimes I_{K'}\right)vec(B_1) =0,\\
				\left(I_{6K} \otimes\rho'(g)-\left(P(g\otimes g)P^{\dag}\otimes\rho(g)\right)^T\otimes I_{K'}\right)vec(B_2) =0,\\
			\end{array}
			\right.
			\label{msystem}
		\end{equation}
	\end{scriptsize}
	where
	\begin{equation*}
		\left\{
		\begin{array}{l}
			B_0 = A_0,\\
			B_1 = \left[A_1,A_2,A_3\right],\\
			B_2 =  \left[A_{11},A_{12},A_{13},A_{22},A_{23},A_{33}\right],\\
		\end{array}
		\right.
	\end{equation*}
	\begin{equation*}
		P = \left[
		\begin{array}{ccccccccc}
			1 &0 &0 &0 &0 &0 &0 &0 &0 \\
			0 &1/2 &0 &1/2 &0 &0 &0 &0 &0 \\
			0 &0 &1/2 &0 &0 &0 &1/2 &0 &0 \\
			0 &0 &0 &0 &1 &0 &0 &0 &0 \\
			0 &0 &0 &0 &0 &1/2 &0 &1/2 &0 \\
			0 &0 &0 &0 &0 &0 &0 &0 &1 \\
		\end{array}
		\right],
	\end{equation*}
	$P^{\dag}$ is the Moore-Penrose inverse of $P$, $I_K$ is a $K$-order identity matrix, $\otimes$ denotes the Kronecker product, the superscript ``$T$" denotes the transpose operator, and $vec(B)$ is the vectorization operator that stacks the columns of $B$ into a vector.
	\label{thm}
\end{theorem}

\emph{Proof sketch.}\quad We substitute Eqn. (\ref{mPsi}) into Eqn. (\ref{mequivariance}), and use the coefficient comparison method to deduce the linear constraints.

That is to say, if the coefficients of PDO-based filter $\Psi$ satisfy the constraints (\ref{mbase}) or (\ref{msystem}) for any $g\in \mathcal{G}$, then $\Psi$ is equivariant over $\mathcal{G}$. Since $\Psi$ is naturally translation equivariant, it is easy to verify that $\Psi$ is equivariant over all combinations of rotations and translations \citep{thomas2018tensor,weiler20183d}. Finally, according to the working spaces, we stack multiple layers $\Psi$'s properly, inserted by nonlinearities that do not disturb the equivariance. As a result, we can get PDO-based steerable 3D CNNs, called PDO-s3DCNNs.

\section{3D Rotation Groups and Feature Fields \label{section4}}
In this section, we show how to apply our theory to various 3D rotation groups and their representations. We provide some examples for ease of understanding. We emphasize that our theory can be applied to any representations, including irreps of discrete subgroups, not limited to these given ones.  In all, common 3D rotation groups include the continuous group $SO(3)$ and its discrete subgroups.
\subsection{Discrete Subgroups}
\textbf{Classification of Discrete Subgroups} \quad Up to conjugacy, any finite discrete 3D rotation group is one of the following groups \cite{artin2011algebra}:  
1) $C_N$: the cyclic group of rotations by multiples of $2\pi/N$ about an axis; 
2) $D_N$: the dihedral group of symmetries of a regular $N$-gon;
3) $\mathcal{T}$: the tetrahedral group of $12$ rotational symmetries of a tetrahedron;
4) $\mathcal{O}$: the octahedral/cubic group of $24$ rotational symmetries of an octahedron/cube; 
5) $\mathcal{I}$: the icosahedral group of $60$ rotational symmetries of an icosahedron.

Particularly, $D_2 = \mathcal{V}$ is also called Klein's four-group, which is the smallest non-cyclic group. The schema of groups $\mathcal{V},\mathcal{T},\mathcal{O},$ and $\mathcal{I}$ and their generators are shown in Appendix \ref{appc}. Particularly, two rotation groups being conjugate means that they are the symmetry groups of one polyhedron and its rotated version, respectively. We treat the groups in each conjugate class without any distinction, because 3D data always assume no preferred rotation transformations.

\noindent\textbf{Regular and Quotient Features} \quad The most important group representation of a finite group is called \emph{regular representation}, and the corresponding feature is called \emph{regular feature}. To be specific, each feature vector $f(x)$ is $|\mathcal{G}|$-dimensional and indexed by $\mathcal{G}$, where $|\mathcal{G}|$ denotes the size of $\mathcal{G}$, i.e., the number of elements in $\mathcal{G}$. A transformation $g\in \mathcal{G}$ acts on $f(x)\in \mathbb{R}^{|\mathcal{G}|}$ in the way that it permutes the value $f_{\tilde{g}}(x)$ to $f_{g\tilde{g}}(x)$ for any $\tilde{g}\in\mathcal{G}$. It is easy to derive the regular representation $\rho(g)$ by referring to the Cayley Table.

\emph{Quotient features} are closely related to regular features. Specifically, given a group $\mathcal{G}$ and its subgroup $\mathcal{H}\leq \mathcal{G}$, the left cosets $g\mathcal{H}$ of $\mathcal{H}$ partition $\mathcal{G}$ properly, where $g\in \mathcal{G}$, and we denote the set of left cosets as $\mathcal{G}/\mathcal{H}$. As for $\mathcal{H}$-quotient features, each feature vector $f(x)\in \mathbb{R}^{|\mathcal{G}|/|\mathcal{H}|}$ is indexed by $\mathcal{G}/\mathcal{H}$. A transformation $g\in \mathcal{G}$ acts on $f(x)$ in the way that it permutes the value $f_{\tilde{g}\mathcal{H}}(x)$ to $f_{g\tilde{g}\mathcal{H}}(x)$, and the corresponding group representation $\rho(g)$ can be obtained similarly. If $\mathcal{H}=\{e\}$, where $e$ is the identity element of $\mathcal{G}$, then an $\mathcal{H}$-quotient feature is equivalent to a regular feature, indicating that the regular feature is essentially a special case of the quotient feature.

An advantage of quotient features over regular features is that the number of channels required for each feature is reduced by $|\mathcal{H}|$ times. For example, for the octahedral group $\mathcal{O}$, a regular feature contains $24$ channels. Noting that $\mathcal{V}$ and $\mathcal{T}$ are both subgroups of $\mathcal{O}$, we can employ $\mathcal{V}$- and $\mathcal{T}$-quotient features, then each feature only contains $6$ and $2$ channels, respectively. We emphasize that although it seems that $\mathcal{O}/\mathcal{V}\cong D_3$ and $\mathcal{O}/\mathcal{T}\cong C_2$, where $\cong$ denotes the group isomorphism, employing quotient features does not indicate that the equivariance is reduced to the quotient space, because we still require the equivariance constraints to be satisfied for any $g$ in $\mathcal{G}$.

\noindent\textbf{Equivariant Kernels over Discrete Subgroups}\quad As derived in Section \ref{kernels}, given the input and the output feature fields, solving an equivariant kernel is no more difficult than solving the homogeneous linear equations, Eqn. (\ref{msystem}). However, Eqn. (\ref{msystem}) still needs to be satisfied for any $g\in \mathcal{G}$, and it is time-consuming to solve if $\mathcal{G}$ is very large. Fortunately, it is easy to verify that as long as Eqn. (\ref{mequivariance}) or (\ref{mbase}) is satisfied for the generators of $\mathcal{G}$, $\Psi$ can be an equivariant kernel over $\mathcal{G}$ (see Theorem \ref{thmb2} in Appendix). In this way, the bases of the coefficients $B_0,B_1$ and $B_2$ can be efficiently computed offline using singular value decomposition (SVD). Specifically, in order to solve a linear constraint $Qx=0$, we use SVD to decompose $Q$ as $U\Sigma V^T$ , and
then the non-zero column vectors of $V$ are basic solutions
of the linear constraint, i.e., the bases of equivariant filters. For example, when $\mathcal{G}=\mathcal{O}$, the dimension of each solution space is given in Table \ref{O_base}. Then we parameterize $B_0,B_1$ and $B_2$ as linear combinations of their bases, and the coefficients are learnable. Finally, they are substituted into Eqn. (\ref{mPsi}) to obtain a parameterized equivariant kernel $\Psi$, which is learnable.

\begin{table}[t]
	\scriptsize
	\caption{$(n_{B_0},n_{B_1},n_{B_2})$ denote the dimensions of solution spaces of $B_0, B_1$ and $B_2$. $\rho(g)$ and $\rho'(g)$ denote the group representations of the input and the output feature fields, respectively.}
	\centering
	\begin{tabular}{lcccc}
		\toprule
		\diagbox{$\rho(g)$}{$\rho'(g)$}   &  Trivial & $\mathcal{T}$-quotient  & $\mathcal{V}$-quotient & Regular \\
		\midrule
		Trivial  & $(1,0,1)$ & $(1,0,1)$ & $(1,0,3)$ & $(1,3,6)$\\
		$\mathcal{T}$-quotient & $(1,0,1)$ & $(2,0,2)$ & $(2,0,6)$ & $(2,6,12)$ \\
		$\mathcal{V}$-quotient & $(1,0,3)$ & $(2,0,6)$ & $(6,0,18)$ & $(6,18,36)$\\
		Regular & $(1,3,6)$ & $(2,6,12)$ & $(6,18,36)$ & $(24,72,144)$\\
		\bottomrule
	\end{tabular}
	
	\label{O_base}
\end{table}

\subsection{Continuous Group $SO(3)$}
As for the $SO(3)$ group, we cannot leverage regular or quotient features, because $SO(3)$ is an infinite group. Following SE3CNN \cite{weiler20183d}, we use irreducible features instead, which are realized by irreps. Specifically, any representation of $SO(3)$ can be decomposed into irreps of dimension $2l+1$, for $l=0,1,2,\cdots,\infty$. The irrep acting on the $l$-order irreducible features $f(x)\in \mathbb{R}^{2l+1}$ is known as the Wigner-D matrix of order $l$, denoted as $D^l(\rho)$. Particularly, $D^0(\rho)=(1)$, which exactly corresponds to the trivial representation. Then we solve the linear equations (\ref{msystem}) given the input and the output features. However, the difficulty lies in that Eqn. (\ref{msystem}) essentially contains infinite equations as $SO(3)$ contains infinite elements, which is impossible to solve directly. Fortunately, we have the following observation, and then the equivariant filters can be solved easily.
\begin{theorem}
	If Eqn. (\ref{mequivariance}) or (\ref{mbase}) is satisfied for $g_1=Z(1)$ and $g_2 = Y(1)$, where
		\begin{scriptsize}
	\begin{align*}
		Z(\alpha)=\left[
		\begin{array}{ccc}
			\cos\alpha & -\sin\alpha & 0\\
			\sin\alpha & \cos\alpha & 0\\
			0 & 0 & 1\\
		\end{array}
		\right],Y(\beta)=\left[
		\begin{array}{ccc}
			\cos\beta & 0  & \sin\beta\\
			0  & 1  & 0\\
			-\sin\beta & 0 & \cos\beta
		\end{array}
		\right],
	\end{align*}
		\end{scriptsize}
	then $\forall g\in SO(3)$, Eqn. (\ref{mequivariance}) is satisfied.
	\label{thm3}
\end{theorem}

\emph{Proof intuition.} Any elements in $SO(3)$ can be approximately generated by $Z(1)$ and $Y(1)$.

\section{Implementation \label{implementation}}
We have introduced equivariant 3D filters in the continuous domain. They are easy to implement in the discrete domain by discretizing PDOs using the given discrete data. 
\subsection{Discretization \label{dis}}
\textbf{FD}\quad This is a similar discretization method to that adopted in PDO-eConvs \cite{shen2020pdo}. For the volumetric data/feature $\bm{F}$, it can be viewed as a sample from a continuous function $f$ defined on $\mathbb{R}^3$. Then PDOs acting on $f$ can be discretized as convolutional kernels acting on $\bm{F}$ using FD. Formally, we have that $\forall i\in S_x$,
\begin{equation*}
	\partial_i[f] = u_i \ast \bm{F} + O\left(\rho^2\right),
\end{equation*}
where $S_x= \left\{x_1,x_2,x_3,x_1^2,x_2^2,x_3^2,x_1x_2,x_1x_3,x_2x_3\right\}$, $u_i$'s are the FD schemes of PDOs (see Appendix \ref{appd1}), $\ast$ denotes the convolution operation, and $\rho$ is the diameter of a 3D grid. Since $\Psi$ is essentially a combination of PDOs, it can be discretized as a convolutional filter $\hat \Psi$. Particularly, we observe that all the PDOs up to the second-order can be approximated using $3\times 3\times 3$ convolutional filters with a quadratic precision, so $\hat \Psi$ is also of the size $3\times 3\times 3$, and the equivariance error resulted from FD discretization is of the quadratic order. 

\noindent \textbf{Gaussian Discretization}\quad We can also employ the derivatives of Gaussian function for estimation \cite{jenner2021steerable}: given points $x^n\in \mathbb{R}^3$, $\forall i\in S_x$,
\begin{equation*}
	\partial_i[f](x) \approx \sum_{n=1}^N \partial_iG(x^n;\sigma)f(x+x^n),
\end{equation*}
where $G(x;\sigma)$ is a Gaussian kernel with standard deviation $\sigma$ around 0. 

\noindent\textbf{Extension to Point Clouds}\quad Essentially, PDOs can also be approximated based on irregular point clouds . Specifically, the simplest method to apply our method to 3D point
clouds is using GD, which is easy to implement on irregular
data. We can also estimate partial derivatives using the
Taylor expansion and the least square method, which has
 been successfully implemented in \citep{shen2021pdo}. As a result, our models can also operate on point clouds. In practice, we operate on the volumetric data because they are more general than point cloud data: point clouds can always be voxelized, but not vice versa, such as medical images. Also, regular grids can be processed more efficiently on current hardware.

\subsection{Common Deep Learning Techniques \label{cdlt}}
\textbf{Nonlinearities and Batch Normalization (BN)}\quad  In equivariant models, inserted nonlinearities and BN should be compatible with employed features to preserve equivariance. For trivial, regular, and quotient features, pointwise nonlinearities, e.g., ReLU, are admissible because their representations are realized by permutation matrices. As for BN, it should be implemented with a single scale and a single bias per feature, instead of per channel. In the continuous case, as the irreducible features are not realized by permutation matrices, pointwise nonlinearities and standard BN are not admissible. We instead use the scale-BN and scale-like nonlinearities, e.g., gated nonlinearities, to preserve equivariance (see Appendix \ref{appd2}). However, compared with standard BN, scale-BN cannot correct activations to be zero-mean, and thus is not as  beneficial as standard BN for internal covariate shift reduction and optimization \citep{bjorck2018understanding}. Besides, as proven in \citep{finzi2021practical}, scale-like nonlinearities are not sufficient for universality, and can limit the model performance. These disadvantages also account for why most related works \citep{weiler2019general,jenner2021steerable} employ regular or quotient representations of discrete subgroups for implementation instead of irreps.

\noindent\textbf{Composition of Basic Feature Fields}\quad  Analogy to multiple channels, we can also put multiple basic feature fields together, e.g., regular, quotient and irreducible ones, to acquire more general feature fields (see Appendix \ref{appd3}).

\section{Experiments \label{section6}}
We perform extensive experiments to evaluate the performance of our models. The experimental details for each task are provided in Appendix \ref{appe}.
\subsection{3D Tetris (Testing 3D Equivariance)}
We confirm the equivariance of our models on 3D Tetris, which is a common dataset for testing 3D rotation equivariance \citep{thomas2018tensor,weiler20183d}. We feed the dataset in a single orientation into the network during training, and then test the network with rotated shapes. 

Firstly, we discretize PDOs using FD, and test three $\mathcal{O}$-steerable models with regular, $\mathcal{V}$- and $\mathcal{T}$-quotient features, on 3D Tetris with cubic rotations.  As shown in Table \ref{O}, all the models achieve 100\% test accuracies. Compared with regular features, quotient features have fewer channels, and thus they have a lower computational cost. Although we use a small dataset here, the complexity comparison can be scaled to larger datasets and models. Besides, the $SO(3)$-steerable CNN also achieves 100\% test accuracy under this setting. Theoretically, our $\mathcal{G}$-steerable CNNs with FD discretization can achieve exactly equivariance over the cubic group $\mathcal{O}$ in the discrete domain if $\mathcal{G}$ includes $\mathcal{O}$ as a subgroup (see Theorem \ref{thd1} in Appendix).

\begin{table}[t]
	\small
	\caption{The test accuracy of the $\mathcal{O}$- and $SO(3)$- steerable CNNs discretized by FD on 3D Tetris with cubic rotations.}
	\centering
	\begin{tabular}{ccccc}
		\toprule
		Group & Feature field & Test acc. (\%)  &\# Params & Time \\
		\midrule
		$\mathcal{O}$ & Regular & $100.0\pm 0.0$ & 31k &14.3s\\
		$\mathcal{O}$ & $\mathcal{V}$-quotient &  $100.0\pm 0.0$  & 5.5k & 2.3s\\
		$\mathcal{O}$ & $\mathcal{T}$-quotient &  $100.0\pm 0.0$  &2.2k & 1.3s\\
		\hline
		$SO(3)$ & Irreducible &  $100.0\pm 0.0$ & 22.8k & 66.7s\\
		\bottomrule
	\end{tabular}
	\label{O}
\end{table}

\begin{table}[t]
	\caption{The test accuracy of the $SO(3)$-steerable CNNs on 3D Tetris with arbitrary rotations.}
	\centering
	\begin{tabular}{lccc}
		\toprule
		Discretization & Kernel size& Test acc. (\%)  & Time\\
		\midrule
		FD & $3\times3\times3$ & $18.20\pm 3.13$  & 66.7s \\
		Gaussian & $3\times3\times3$ & $29.99\pm 4.98$ & 67.3s\\
		Gaussian & $5\times5\times5$ & $99.04\pm 0.14$ & 109.5s\\
		\bottomrule
	\end{tabular}
	\label{SO3}
\end{table}


We then test our $SO(3)$-steerable model on the shapes with arbitrary rotations, and FD cannot perform well (see Table \ref{SO3}), because inputs are assumed to be smooth while real data are always non-smooth especially on the edge. Considering that Gaussian functions can help smooth features, we then employ Gaussian discretization and find that it performs very well when using a kernel size of $5\times 5\times 5$. Also, we provide the equivariance error analysis in Appendix \ref{appe1}.

\subsection{SHREC'17 Retrieval  \label{shrec17}}
The SHREC'17 retrieval task \citep{savva2017large} contains 51,162 models of 3D shapes belonging to 55 classes. This dataset is divided into 35,764 training samples, 5,133 validation samples, and 10,265 test samples. We focus on the ``perturbed'' version where models are arbitrarily rotated. The retrieval performance is given by the average value of the mean average precisions (mAP) of micro-average version and macro-average version, denoted as the \emph{score}. We convert these 3D models into voxels of size $64\times 64\times 64$.


\begin{table}[t]
	\caption{The retrieval performance of $\mathcal{V}$-, $\mathcal{T}$-, $\mathcal{O}$-, $\mathcal{I}$- and $SO(3)$-steerable CNNs, tested on SHREC’17.}
	\centering
	\begin{tabular}{cccc}
		\toprule
		Group & Discretization & Feature field & Score \\
		\midrule
		$\mathcal{V}$ &  FD & Regular & 52.7\\
		$\mathcal{T}$ &  FD & Regular & 57.6 \\
		$\mathcal{O}$ &  FD & Regular & \textbf{58.6} \\
		$\mathcal{I}$ &  Gaussian & Regular & 55.5\\
		\hline
		$SO(3)$ & FD & Irreducible & 57.4 \\
		$SO(3)$ & Gaussian  & Irreducible & 58.3 \\
		\bottomrule
	\end{tabular}
	\label{5scores}
\end{table}
\begin{table*}[t]
	\small
	\caption{SHREC' 17 perturbed dataset results. Mixed features mean that the features are composed of regular and $\mathcal{V}$-quotient features. Score is the average value of mAP of micro-average version and macro-average version.} \smallskip
	\centering
	\begin{tabular}{lccccccccc}
		\toprule
		&& \multicolumn{3}{c}{micro} & \multicolumn{3}{c}{macro} & &  \\
		\cmidrule(r){3-5} \cmidrule(r){6-8}
		Method  &  Score & P@N & R@N & mAP & P@N & R@N & mAP  & \#Params & Input\\
		\midrule
		RI-GCN \cite{kim2020rotation} &56.2&69.1&68.0&64.5&47.4&57.0&47.8& 4.4M  &point clouds\\
		 \cite{li2021closer} (with TTA) & 56.5 & 69.4 & 69.4 & 65.8 & 48.1 & 56.0 & 47.2& 2.9M &point clouds\\
		\hline
		S2CNN \cite{cohen2018spherical} & -  & 70.1 & 71.1 & 67.6 & - & - & - & 1.4M & spherical\\
		FFS2CNN \cite{kondor2018clebsch} & - & 70.7 & 72.2 & 68.3 & -&- &- &- & spherical \\
		VolterraNet \cite{banerjee2020volterranet} & - & 71.0 & 70.0 & 67.0 & - & - & - &  0.4M & spherical  \\
		\cite{esteves2018learning} & 56.5 & 71.7 & \textbf{73.7} & 68.5 & 45.0 & 55.0 & 44.4 & 0.5M & spherical\\
		\cite{cobb2020efficient} & - & 71.9 & 71.0 & 67.9 & - & - & - &  0.25M & spherical \\
		\hline
		SE3CNN \cite{weiler20183d} & 55.5 & 70.4 & 70.6 & 66.1 & 49.0 & 54.9 & 44.9  & 0.14M & voxels\\
		\hline
		Ours ($SO(3)$) & 58.3 & 73.1 & 73.4 & 69.3 & \textbf{52.5} & 55.4 & 47.3 & 0.15M & voxels\\		
		Ours ($\mathcal{O}$ with regular features) & 58.6 & 72.9 & 73.0 & 68.8 & 51.9 & 57.7 & 48.3 & 0.15M & voxels\\
		Ours ($\mathcal{O}$ with $\mathcal{V}$-quotient features) &55.5 & 69.2 & 69.6 & 65.0 & 48.0 & 56.3 & 46.0 & 0.15M & voxels\\
		Ours ($\mathcal{O}$ with mixed features) & \textbf{59.1} & \textbf{73.2} & 73.3 & \textbf{69.3} & 51.7 & \textbf{57.8} & \textbf{48.8}  & 0.15M & voxels\\
		
		\bottomrule
	\end{tabular}
	\label{detailed}
\end{table*}

Firstly, we investigate the impact of the equivariance group. We evaluate the performance of $\mathcal{V}$-, $\mathcal{T}$-, $\mathcal{O}$-, and $\mathcal{I}$-steerable CNNs with regular features. The $\mathcal{I}$-steerable one is discretized using Gaussian discretization as its equivariance exceeds the symmetries of regular 3D grids, and others using FD. We adjust the feature numbers so that these models use comparable numbers of parameters (about 0.15M) for a fair comparison. As shown in Table \ref{5scores}, the $\mathcal{O}$-steerable model performs better than $\mathcal{V}$- and $\mathcal{T}$-steerable ones, as it can exploit more symmetries. However, for the $\mathcal{I}$-steerable CNN, we find it too large (each feature contains $60$ channels) and slow to converge during training, and it performs inferiorly to the $\mathcal{O}$-steerable one. We then evaluate $SO(3)$-steerable models with irreducible features. When using FD for discretization, it performs significantly inferiorly to the $\mathcal{O}$-steerable one even though it can also achieve exact equivariance over $\mathcal{O}$. A similar phenomenon has also occurred in E2CNNs \citep{weiler2019general} in the 2D case, where $C_N$-steerable models significantly outperform $SO(2)$-steerable models. We argue that this is because irreducible features are not amenable to conventional BN and pointwise nonlinearities, which degrades the model performance in practice. Empirically, we evaluate an $\mathcal{O}$-steerable model using the special BNs and nonlinearities for irreps, and they
perform worse than standard ones (56.5 vs. 58.6), which verifies
their inferiority. When using Gaussian discretization, the performance gets better due to the improved equivariance, but the disadvantage still exists (58.3 vs. 58.6). In addition, the larger kernel size ($5\times 5 \times 5$) results in a larger computational cost.

Then we compare our methods to various equivariant 3D models, as shown in Table \ref{detailed}. Our $SO(3)$-steerable model performs much better than its counterpart SE3CNN using comparable numbers of parameters, showing great expressive ability of our PDO-based filters. Compared with some recent point clouds based rotation equivariant models \cite{kim2020rotation,chen2021equivariant,li2021closer}, our method performs better, as volumetric data are more regular and can provide more structured and compact information.  Also, our method shows great parameter efficiency (0.15M vs. 2.9M+). In addition, our method outperforms some equivariant spherical CNNs \cite{cohen2018spherical,kondor2018clebsch,esteves2018learning,banerjee2020volterranet,cobb2020efficient}, where 3D shapes are projected to a sphere to achieve rotation equivariance. Spherical models perform worse as they are not translation equivariant while ours are.

In addition, we examine the effectiveness of quotient features. The $\mathcal{O}$-steerable model using $\mathcal{V}$-quotient features performs inferiorly to the model using regular features (55.5 vs. 58.6). It is probably because quotient features cannot differentiate the transformations in one coset in the channel domain as they relate to the same representation, which could result in the weak performance. In addition, we notice that the filters between quotient features essentially do not employ the first order PDOs, which may also relate to their weak expressivity. As shown in Table \ref{O_base}, when the input and the output feature fields are both quotient features, the coefficients of the first order PDOs, $B_1$, can only be $\bm{0}$. An ablation study to justify the choice of PDOs is in Appendix \ref{appe2}. When we employ a feature field composed of regular and quotient features, the performance can be drastically improved (from 55.5 to 59.1),  indicating that quotient features can be effective supplements for the regular features. Essentially, regular features of $\mathcal{G}$-steerable models are redundant for
some $\mathcal{H}$-invariant patterns ($\mathcal{H}\leq\mathcal{G}$), so incorporating some
H-quotient features can help relieve this redundancy and
compress the model, which is beneficial for model learning.

\subsection{ISBI 2012 Challenge}
%
%
%
%
%
\begin{table}[t]
	\caption{The performance on the ISBI 2012 Challenge evaluated and compared by $V_{\text{rand}}$ and $V_{\text{info}}$, which are explained in \cite{arganda2015crowdsourcing}. Larger is better. Different models are primarily compared by $V_{\text{rand}}$, as it is more robust than $V_{\text{info}}$.}
	\centering
		\small
	\begin{tabular}{lcc}
	\toprule
	Method & $V_{\text{rand}}$  & $V_{\text{info}}$\\
	\midrule
	U-Net \citep{ronneberger2015u}  & 0.97276  & 0.98662\\
	FusionNet \citep{quan2016fusionnet}& 0.97804  & 0.98995\\
	CubeNet \citep{worrall2018cubenet} & 0.98018  & 0.98202\\
	SFCNN \cite{weiler2018learning}& 0.98680  & \textbf{0.99144}\\
	\hline
	PDO-s3DCNN ($\mathcal{V}$) &0.98415 & 0.99031\\
	PDO-s3DCNN ($\mathcal{O}$) &\textbf{0.98727} &  0.99089\\
	\bottomrule
	\end{tabular}
	\label{ISBI}
\end{table}

ISBI 2012 Challenge \cite{arganda2015crowdsourcing} is a volumetric boundary segmentation benchmark, and the target is to segment Drosophila ventral nerve cords from a serial-section transmission electron microscopy image. This dataset is suitable for evaluating a rotation-equivariant model, because biomedical images always have no inherent orientations. The full training image is $512\times 512\times 30$ voxels in shape, and each voxel is $4\times 4\times 50$ nm$^3$. The setting is the same for test
images. Following the official ranking list, we compare different models using the metrics $V_{\text{rand}}$ and $V_{\text{info}}$, which is explained in \citep{arganda2015crowdsourcing}. Larger is better. Different models are primarily compared by $V_{\text{rand}}$, as it is empirically more robust than $V_{\text{info}}$.

Firstly, we evaluate the $\mathcal{V}$-steerable model. We replace the equivariant filters in CubeNet \citep{worrall2018cubenet} by ours with regular features. We keep the architecture and the numbers of features the same as that in CubeNet to maintain a fair comparison. As shown in Table \ref{ISBI}, our model performs better than CubeNet with a much lower network complexity (4.4M vs. 11.9M), noting that on average each learnable filter of our model contains 10 parameters, while that of CubeNet contains 27, which shows a great storage advantage. We then employ an $\mathcal{O}$-steerable model (use comparable numbers of parameters with the $\mathcal{V}$-steerable one) and achieve a better result, because it can exploit more symmetries. Particularly, CubeNet cannot exploit $\mathcal{O}$-equivariance, because the voxels here are not cubic. Our model even outperforms SFCNN \cite{weiler2018learning} (0.98727 vs. 0.98680 in $V_{\text{rand}}$), which additionally uses a task-specific lifting multi-cut \citep{beier2016efficient} post-processing to improve performance, while we do not use any post-processing.

\section{Conclusions \label{conclusions}}
We employ PDOs to establish general 3D steerable CNNs, which cover both the continuous group $SO(3)$ and its discrete subgroups. Our theoretical framework can also deal with various feature fields easily, such as regular, quotient, and irreducible features. Experiments verify that our methods can preserve equivariance well in the discrete domain and perform better than previous works with a low network complexity.

Actually, besides volumetric data, our method can also operate on the point cloud data. In addition, our theoretical framework can be applied to arbitrary representations, while we only work on several given representations in practice. We will explore more possibilities in future work.

\section*{Acknowledgements}

This work was supported by the National Key Research and Development Program of China under grant 2018AAA0100205. Z. Lin was supported by the NSF China (No. 61731018), NSFC Tianyuan Fund for Mathematics (No. 12026606), and Project 2020BD006 supported by PKU-Baidu Fund.

\bibliography{PDO-s3DCNNs}
\bibliographystyle{icml2022}

\newpage
\clearpage
\appendix
\section{Groups and Group Representations\label{appa}}
\textbf{Groups}\quad
\\

\begin{itemize}
	\item Closure: $\forall g_1,g_2\in \mathcal{G}, g_1\circ g_2\in\mathcal{G}$;
	\item Associativity: $\forall g_1,g_2,g_3\in \mathcal{G}, (g_1\circ g_2)\circ g_3 = g_1\circ (g_2 \circ g_3)$;
	\item Identity: There exists an identity element $e\in\mathcal{G}$ such that $e\circ g=g\circ e=g$ for all $g\in \mathcal{G}$;
	\item Inverses: $\forall g\in\mathcal{G}$, there exists a $g^{-1}\in\mathcal{G}$, such that $g^{-1}\circ g=g\circ g^{-1}=e$.
\end{itemize}
In practice, we always omit writing the binary composition operator $\circ$, so we would write $gh$ instead of $g\circ h$. Groups can be finite or infinite.

\noindent\textbf{Group Representations}\quad A group representation $\rho:\mathcal{G}\rightarrow GL(N)$ is a mapping from a group $\mathcal{G}$ to the set of $N\times N$ invertible matrices $GL(N)$. Critically, $\rho$ satisfies the following property:
\begin{equation*}
	\forall g_1,g_2\in \mathcal{G}, \quad\rho\left(g_1g_2\right)=\rho\left(g_1\right)\rho\left(g_2\right).
\end{equation*}

\noindent\textbf{Group Isomorphism}\quad
Given two groups $\mathcal{G}_1$ and $\mathcal{G}_2$, the two groups are isomorphic if there exists an isomorphism from one to the other, denoted as $\mathcal{G}_1\cong\mathcal{G}_2$. Spelled out, this means that there is a bijective function $\sigma:\mathcal{G}_1\to\mathcal{G}_2$ such that
\begin{equation*}
	\forall u,v\in \mathcal{G}_1, \quad \sigma(uv)=\sigma(u)\sigma(v).
\end{equation*}

\section{The Proofs of Theorems\label{appb}}

We define a convolution on the input function $f\in C^{\infty}(\mathbb{R}^3,\mathbb{R}^{K})$:
\begin{align}
	\Psi[f] =& \left(A_0+\,A_1\partial_{x_1}+A_2\partial_{x_2}+A_3\partial_{x_3}+A_{11}\partial_{x_1^2}\right.\notag\\
	&\left.+A_{12}\partial_{x_1x_2}+A_{13}\partial_{x_1x_3}+A_{22}\partial_{x_2^2}+A_{23}\partial_{x_2x_3}\right.\notag\\
	&\left.+A_{33}\partial_{x_3^2}\right)[f],
	\label{Psi}
\end{align}
where the coefficients $A_i\in \mathbb{R}^{K’\times K}$. In order that $\Psi$ is an equivariant mapping over $\mathcal{G}$, it should satisfy the equivariance condition: $\forall g\in \mathcal{G}$,
\begin{equation}
	\pi'(g)\left[\Psi[f]\right]=\Psi\left[\pi(g)[f]\right],\label{equivariance}
\end{equation}
where
\begin{equation*}
	\left\{
	\begin{array}{l}
		[\pi(g)f](x) = \rho(g)f\left(g^{-1}x\right),\\
		\left[\pi'(g)f\right](x) = \rho'(g)f\left(g^{-1}x\right).\\
	\end{array}
	\right.
	\label{rho}
\end{equation*}
To solve this equivariance requirement, we have the following theorem.

\begin{theorem}
	$\Psi$ is equivariant over $\mathcal{G}$, if and only if its coefficients satisfy the following linear constraints: $\forall g\in \mathcal{G}$,
	\begin{equation}
		\left\{
		\begin{array}{l}
			\rho'(g)B_0 = B_0\rho(g),\\
			\rho'(g)B_1 =B_1\left(g\otimes \rho(g)\right),\\
			\rho'(g)B_2 =B_2\left(P\left(g\otimes g\right)P^{\dag}\otimes\rho(g)\right),
		\end{array}
		\right.
		\label{base}
	\end{equation}
	i.e.,
	\begin{scriptsize}
		\begin{equation}
			\left\{
			\begin{array}{l}
				\left(I_K \otimes\rho'(g)-\rho(g)^T\otimes I_{K'}\right)vec(B_0) =0,\\
				\left(I_{3K} \otimes\rho'(g)-\left(g\otimes \rho(g)\right)^T\otimes I_{K'}\right)vec(B_1) =0,\\
				\left(I_{6K} \otimes\rho'(g)-\left(P(g\otimes g)P^{\dag}\otimes\rho(g)\right)^T\otimes I_{K'}\right)vec(B_2) =0,\\
			\end{array}
			\right.
			\label{system}
		\end{equation}
	\end{scriptsize}
	where
	\begin{equation*}
		\left\{
		\begin{array}{l}
			B_0 = A_0,\\
			B_1 = \left[A_1,A_2,A_3\right],\\
			B_2 =  \left[A_{11},A_{12},A_{13},A_{22},A_{23},A_{33}\right],\\
		\end{array}
		\right.
	\end{equation*}
	\begin{equation*}
		P = \left[
		\begin{array}{ccccccccc}
			1 &0 &0 &0 &0 &0 &0 &0 &0 \\
			0 &1/2 &0 &1/2 &0 &0 &0 &0 &0 \\
			0 &0 &1/2 &0 &0 &0 &1/2 &0 &0 \\
			0 &0 &0 &0 &1 &0 &0 &0 &0 \\
			0 &0 &0 &0 &0 &1/2 &0 &1/2 &0 \\
			0 &0 &0 &0 &0 &0 &0 &0 &1 \\
		\end{array}
		\right],
	\end{equation*}
	$P^{\dag}$ is the Moore-Penrose inverse of $P$, $I_K$ is a $K$-order identity matrix, $\otimes$ denotes the Kronecker product, the superscript ``$T$" denotes the transpose operator, and $vec(B)$ is the vectorization operator that stacks the columns of $B$ into a vector.
	\label{thm}
\end{theorem}

\begin{proof}
	we substitute (\ref{Psi}) into the equivariance condition Eqn. (\ref{equivariance}) and get that for any $ g\in \mathcal{G}$,
	\begin{align}
		&\rho'(g)\left(A_0+A_1\partial_{x_1}+A_2\partial_{x_2}+A_3\partial_{x_3}+A_{11}\partial_{x_1^2}\right.\notag\\
		&\left.+A_{12}\partial_{x_1x_2}+A_{13}\partial_{x_1x_3}+A_{22}\partial_{x_2^2}+A_{23}\partial_{x_2x_3}\right.\notag\\
		&\left.+A_{33}\partial_{x_3^2}\right)[f]\left(g^{-1}x\right)\notag\\
		=&\left(A_0+A_1\partial_{x_1}+A_2\partial_{x_2}+A_3\partial_{x_3}+A_{11}\partial_{x_1^2}\right.\notag\\
		&\left.+A_{12}\partial_{x_1x_2}+A_{13}\partial_{x_1x_3}+A_{22}\partial_{x_2^2}+A_{23}\partial_{x_2x_3}\right.\notag\\
		&\left.+A_{33}\partial_{x_3^2}\right)\left[\rho(g)f\left(g^{-1}x\right)\right].
		\label{extend}
	\end{align}	
	We denote that $\forall i = 1,2,\cdots,K,$
	\begin{align*}
		\nabla^T \left[f_i\right]=& \left[\partial_{x_1}\left[f_i\right],\partial_{x_2}\left[f_i\right],\partial_{x_3}\left[f_i\right]\right],\\
		\nabla^2 \left[f_i\right]=&\left[
		\begin{array}{ccc}
			\partial_{x_1^2}\left[f_i\right] & \partial_{x_1x_2}\left[f_i\right]& \partial_{x_1x_3}\left[f_i\right]\\
			\partial_{x_1x_2}\left[f_i\right] & \partial_{x_2^2}\left[f_i\right] & \partial_{x_2x_3}\left[f_i\right]\\	 
			\partial_{x_1x_3}\left[f_i\right] & \partial_{x_2x_3}\left[f_i\right] & \partial_{x_3^2}\left[f_i\right]\\	 
		\end{array}
		\right].
	\end{align*}
	Then the left hand side of Eqn. (\ref{extend}) can be written as
	\begin{equation}
		\rho'(g)\left(B_0f\left(g^{-1}x\right)+B_1 V_1+\bar{B}_2\bar{V}_2\right)\label{1lfh},
	\end{equation}
	where
	\begin{align*}
		\bar{B}_2 =& \left[A_{11},\frac{A_{12}}{2},\frac{A_{13}}{2},\frac{A_{12}}{2},A_{22},\frac{A_{23}}{2},\frac{A_{13}}{2},\frac{A_{23}}{2},A_{33}\right],
	\end{align*}
	\begin{small}
		\begin{align*}
			V_1=&vec\left[
			\begin{array}{c}
				\nabla^T \left[f_1\right]\left(g^{-1}x\right)\\
				\nabla^T \left[f_2\right]\left(g^{-1}x\right)\\
				\cdots\\
				\nabla^T\left [f_K\right]\left(g^{-1}x\right)\\
			\end{array}
			\right]\\
			=&
			vec\left[\partial_{x_1}[f]\left(g^{-1}x\right),\partial_{x_2}[f]\left(g^{-1}x\right),\partial_{x_3}[f]\left(g^{-1}x\right)\right],\\
		\end{align*}
	\end{small}
	and
	\begin{small}
		\begin{align*}
			\bar{V}_2 = &vec\left[
			\begin{array}{c}
				vec^T\left(\nabla^2 \left[f_1\right]\left(g^{-1}x\right)\right)\\
				vec^T\left(\nabla^2 \left[f_2\right]\left(g^{-1}x\right)\right)\\
				\cdots\\
				vec^T\left(\nabla^2 \left[f_K\right]\left(g^{-1}x\right)\right)\\
			\end{array}
			\right]\\
			=&vec
			\left[
			\partial_{x_1^2}[f]\left(g^{-1}x\right),\partial_{x_1x_2}[f]\left(g^{-1}x\right),\partial_{x_1x_3}[f]\left(g^{-1}x\right),\right.\\
			&\left.	\partial_{x_1x_2}[f]\left(g^{-1}x\right),\partial_{x_2^2}[f]\left(g^{-1}x\right),\partial_{x_2x_3}[f]\left(g^{-1}x\right),\right.\\
			&\left.\partial_{x_1x_3}[f]\left(g^{-1}x\right),\partial_{x_2x_3}[f]\left(g^{-1}x\right),\partial_{x_3^2}[f]\left(g^{-1}x\right)\right].
		\end{align*}
	\end{small}
	In addition, we have
	\begin{equation*}
		\bar{B}_2=B_2\left(P\otimes I_K\right).
	\end{equation*}
	Then, we denote that
	\begin{small}
		\begin{align*}
			V_2 =&vec
			\left[
			\partial_{x_1^2}[f]\left(g^{-1}x\right),\partial_{x_1x_2}[f]\left(g^{-1}x\right),\partial_{x_1x_3}[f]\left(g^{-1}x\right),\right.\\
			&\left.\partial_{x_2^2}[f]\left(g^{-1}x\right),\partial_{x_2x_3}[f]\left(g^{-1}x\right),\partial_{x_3^2}[f]\left(g^{-1}x\right)\right],
		\end{align*}
	\end{small}
	and we have 
	\begin{equation*}
		\bar{V}_2 = \left(P^{\dag}\otimes I_K\right)V_2,
	\end{equation*}
	where 
	\begin{equation*}
		P^{\dag} = \left[
		\begin{array}{ccccccccc}
			1 &0 &0 &0 &0 &0 \\
			0 &1 &0 &0 &0 &0 \\		
			0 &0 &1 &0 &0 &0 \\
			0 &1 &0 &0 &0 &0 \\
			0 &0 &0 &1 &0 &0 \\
			0 &0 &0 &0 &1 &0 \\
			0 &0 &1 &0 &0 &0 \\
			0 &0 &0 &0 &1 &0 \\
			0 &0 &0 &0 &0 &1 \\						
		\end{array}
		\right],
	\end{equation*}
	and it is easy to verify that $P^{\dag}$ is exactly the Moore-Penrose inverse of $P$. As a result, Eqn. (\ref{1lfh}) can be further written as 
	\begin{small}
		\begin{align*}
			&\rho'(g)\left(B_0f\left(g^{-1}x\right)+B_1 V_1+\bar{B}_2\bar{V}_2\right)\\
			=&\rho'(g)\left(B_0f\left(g^{-1}x\right)+B_1 V_1+B_2\left(P\otimes I_K\right)\left(P^{\dag}\otimes I_K\right)V_2\right)\\
			=&\rho'(g)\left(B_0f\left(g^{-1}x\right)+B_1 V_1+B_2V_2\right).\\
		\end{align*}
	\end{small}
	Also, the right hand side of (\ref{extend}) can be written as \footnote{In the following derivation, we utilize two important properties of Kronecker product: (1) $ vec(MXN)=\left(N^T\otimes M\right)vec(X)$; (2) $\left(M_1\otimes N_1\right)\left(M_2 \otimes N_2\right)=\left(M_1M_2\otimes N_1N_2\right)$. Besides, $g^{-1}=g^T$, because $g$ is essentially an orthogonal matrix.}
	\begin{scriptsize}
		\begin{align*}
			&\left(A_0+A_1\partial_{x_1}+A_2\partial_{x_2}+A_3\partial_{x_3}+A_{11}\partial_{x_1^2}\right.\\
			&\left.+A_{12}\partial_{x_1x_2}+A_{13}\partial_{x_1x_3}+A_{22}\partial_{x_2^2}+A_{23}\partial_{x_2x_3}\right.\notag\\
			&\left.+A_{33}\partial_{x_3^2}\right)\left[\rho(g)f\left(g^{-1}x\right)\right]\\
			=	&\left(A_0\rho(g)+A_1\rho(g)\partial_{x_1}+A_2\rho(g)\partial_{x_2}+A_3\rho(g)\partial_{x_3}\right.\\
			&\left.+A_{11}\rho(g)\partial_{x_1^2}+A_{12}\rho(g)\partial_{x_1x_2}+A_{13}\rho(g)\partial_{x_1x_3}\right.\\
			&\left.+A_{22}\rho(g)\partial_{x_2^2}+A_{23}\rho(g)\partial_{x_2x_3}+A_{33}\rho(g)\partial_{x_3^2}\right)\left[f\left(g^{-1}x\right)\right]\\
			=&B_0\rho(g)f\left(g^{-1}x\right)+B_1\left(I_3\otimes \rho(g)\right) vec\left[
			\begin{array}{c}
				\nabla^T \left[f_1\left(g^{-1}x\right)\right]\\
				\nabla^T \left[f_2\left(g^{-1}x\right)\right]\\
				\cdots\\
				\nabla^T \left[f_K\left(g^{-1}x\right)\right]\\
			\end{array}
			\right]\\
			&+\bar{B}_2\left(I_9\otimes\rho(g)\right)vec\left[
			\begin{array}{c}
				vec^T\left(\nabla^2 \left[f_1\left(g^{-1}x\right)\right]\right)\\
				vec^T\left(\nabla^2 \left[f_2\left(g^{-1}x\right)\right]\right)\\
				\cdots\\
				vec^T\left(\nabla^2 \left[f_K\left(g^{-1}x\right)\right]\right)\\
			\end{array}
			\right]\\
			=	&B_0\rho(g)f\left(g^{-1}x\right)+B_1\left(I_3\otimes \rho(g)\right) vec\left[
			\begin{array}{c}
				\nabla^T \left[f_1\right]\left(g^{-1}x\right)g^{-1}\\
				\nabla^T \left[f_2\right]\left(g^{-1}x\right)g^{-1}\\
				\cdots\\
				\nabla^T \left[f_K\right]\left(g^{-1}x\right)g^{-1}\\
			\end{array}
			\right]\\
			&+\bar{B}_2\left(I_9\otimes\rho(g)\right)vec\left[
			\begin{array}{c}
				vec^T\left(g\nabla^2 \left[f_1\right]\left(g^{-1}x\right)g^{-1}\right)\\
				vec^T\left(g\nabla^2 \left[f_2\right]\left(g^{-1}x\right)g^{-1}\right)\\
				\cdots\\
				vec^T\left(g\nabla^2 \left[f_K\right]\left(g^{-1}x\right)g^{-1}\right)\\
			\end{array}
			\right]\\
			=	&B_0\rho(g)f\left(g^{-1}x\right)\\
			&+B_1\left(I_3\otimes \rho(g)\right) vec\left[I_K\left(
			\begin{array}{c}
				\nabla^T \left[f_1\right]\left(g^{-1}x\right)\\
				\nabla^T \left[f_2\right]\left(g^{-1}x\right)\\
				\cdots\\
				\nabla^T \left[f_K\right]\left(g^{-1}x\right)\\
			\end{array}
			\right)g^{-1}
			\right]\\		
			&+\bar{B}_2(I_9\otimes\rho(g))vec\left[
			\begin{array}{c}
				((g\otimes g)vec(\nabla^2 [f_1](g^{-1}x)))^T\\
				((g\otimes g)vec(\nabla^2 [f_2](g^{-1}x)))^T\\
				\cdots\\
				((g\otimes g)vec(\nabla^2 [f_K](g^{-1}x)))^T\\
			\end{array}
			\right]\\
			=	&B_0\rho(g)f(g^{-1}x)+B_1(I_3\otimes \rho(g)) (g\otimes I_K) vec\left[
			\begin{array}{c}
				\nabla^T [f_1](g^{-1}x)\\
				\nabla^T [f_2](g^{-1}x)\\
				\cdots\\
				\nabla^T [f_K](g^{-1}x)\\
			\end{array}
			\right]\\
			&+\bar{B}_2(I_9\otimes\rho(g))vec\left[I_K\left(
			\begin{array}{c}
				vec^T\left(\nabla^2 [f_1]\left(g^{-1}x\right)\right)\\
				vec^T\left(\nabla^2 [f_2]\left(g^{-1}x\right)\right)\\
				\cdots\\
				vec^T\left(\nabla^2 [f_K]\left(g^{-1}x\right)\right)\\
			\end{array}
			\right)
			(g\otimes g)^T
			\right]\\
			=	&B_0\rho(g)f(g^{-1}x)+B_1(I_3\otimes \rho(g)) (g\otimes I_K) vec\left[
			\begin{array}{c}
				\nabla^T [f_1]\left(g^{-1}x\right)\\
				\nabla^T [f_2]\left(g^{-1}x\right)\\
				\cdots\\
				\nabla^T [f_K]\left(g^{-1}x\right)\\
			\end{array}
			\right]\\
			&+\bar{B}_2(I_9\otimes\rho(g))(g\otimes g \otimes I_K)vec\left[
			\begin{array}{c}
				vec^T\left(\nabla^2 [f_1]\left(g^{-1}x\right)\right)\\
				vec^T\left(\nabla^2 [f_2]\left(g^{-1}x\right)\right)\\
				\cdots\\
				vec^T\left(\nabla^2 [f_K]\left(g^{-1}x\right)\right)\\
			\end{array}
			\right]\\
			=	&B_0\rho(g)f\left(g^{-1}x\right)+B_1\left(g\otimes \rho(g)\right)V_1\\
			&+B_2\left(P\otimes I_K\right)\left(I_9\otimes\rho(g)\right)\left(g\otimes g \otimes I_K\right)\left(P^{\dag}\otimes I_K\right)V_2\\
			=	&B_0\rho(g)f\left(g^{-1}x\right)+B_1\left(g\otimes \rho(g)\right)V_1\\
			&+B_2\left(P\left(g\otimes g\right)P^{\dag}\otimes\rho(g)\right)V_2.\\
		\end{align*}
	\end{scriptsize}
	As a result, Eqn. (\ref{extend}) can be deduced to
	\begin{scriptsize}
		\begin{align*}
			&\rho'(g)B_0f\left(g^{-1}x\right)+\rho'(g)B_1 V_1+\rho'(g)B_2V_2\\
			=&B_0\rho(g)f\left(g^{-1}x\right)+B_1\left(g\otimes \rho(g)\right)V_1+B_2\left(P\left(g\otimes g\right)P^{\dag}\otimes\rho(g)\right)V_2
		\end{align*}
	\end{scriptsize}
	Finally, we can use the coefficient comparison method to derive that Eqn. (\ref{extend}) is satisfied if and only if
	\begin{equation*} 
		\left\{
		\begin{array}{l} 
			\rho'(g)B_0 = B_0\rho(g),\\
			\rho'(g)B_1 =B_1\left(g\otimes \rho(g)\right),\\ 
			\rho'(g)B_2 =B_2\left(P\left(g\otimes g\right)P^{\dag}\otimes\rho(g)\right),
		\end{array}
		\right.
	\end{equation*} 
	i.e., 
	\begin{scriptsize}
		\begin{equation*}
			\left\{
			\renewcommand{\arraystretch}{2.0}
			\begin{array}{l}
				\left(I_K \otimes\rho'(g)-\rho(g)^T\otimes I_{K'}\right)vec\left(B_0\right) =0,\\
				\left(I_{3K} \otimes\rho'(g)-\left(g\otimes \rho(g)\right)^T\otimes I_{K'}\right)vec\left(B_1\right) =0,\\
				\left(I_{6K} \otimes\rho'(g)-\left(P(g\otimes g)P^{\dag}\otimes\rho(g)\right)^T\otimes I_{K'}\right)vec\left(B_2\right) =0,\\
			\end{array}
			\right.
			\label{system}
		\end{equation*}
	\end{scriptsize}
	\label{proof1}
\end{proof}

\begin{theorem}
	When $\mathcal{G}$ is a discrete group, if Eqn. (\ref{equivariance}) or (\ref{base}) is satisfied for the generators of $\mathcal{G}$, then $\forall g\in \mathcal{G}$, Eqn. (\ref{equivariance}) is satisfied.
	\label{thmb2}
\end{theorem}
\begin{proof}
	Firstly, for any given $g\in \mathcal{G}$, it is easy to deduce that Eqn. (\ref{equivariance}) is equivalent to Eqn. (\ref{base}) from Theorem \ref{thm}. Suppose that $g_1$ and $g_2$ are the generators of $\mathcal{G}$, we have
	\begin{align*}
		\pi'\left(g_1\right)[\Psi[f]]=&\Psi\left[\pi\left(g_1\right)[f]\right],\\
		\quad \pi'\left(g_2\right)[\Psi[f]]=&\Psi\left[\pi\left(g_2\right)[f]\right].
	\end{align*}
	As a result, $\forall m\in \mathbb{Z}$,
	\begin{align*}
		\pi'\left(g_1^m\right)\left[\Psi[f]\right]
		=&\pi'\left(g_1\right)^m\left[\Psi[f]\right]\\
		=&\pi'\left(g_1\right)^{m-1}\pi'\left(g_1\right)\left[\Psi[f]\right]\\
		=&\pi'\left(g_1\right)^{m-1}\Psi\left[\pi(g_1)[f]\right],\\
		=&\Psi\left[\pi(g_1)^m[f]\right],\\
		=&\Psi\left[\pi\left(g_1^m\right)[f]\right].\\
	\end{align*}
	Analogously, we have
	\begin{equation*}
		\pi'\left(g_2^m\right)\left[\Psi[f]\right]=\Psi\left[\pi\left(g_2^m\right)[f]\right].
	\end{equation*}
	Actually, any $g\in \mathcal{G}$ can be generated by $g_1$ and $g_2$, i.e., it can be written as $g=g_1^{m_1}g_2^{m_2}\cdots g_1^{m_{l-1}}g_2^{m_l}$, where $m_1,m_2,\cdots,m_l\in \mathbb{Z}$. Consequently, we have
	\begin{align*}
		&\pi'(g)[\Psi[f]]\\
		=&\pi'\left(g_1^{m_1}g_2^{m_2}\cdots g_1^{m_{l-1}}g_2^{m_l}\right)[\Psi[f]]\\
		=&\pi'\left(g_1^{m_1}\right)\pi'\left(g_2^{m_2}\right)\cdots \pi'\left(g_1^{m_{l-1}}\right)\pi'\left(g_2^{m_l}\right)[\Psi[f]]\\
		=&\Psi\left[\pi\left(g_1^{m_1}\right)\pi\left(g_2^{m_2}\right)\cdots \pi\left(g_1^{m_{l-1}}\right)\pi\left(g_2^{m_l}\right)[f]\right]\\
		=&\Psi\left[\pi\left(g_1^{m_1}g_2^{m_2}\cdots g_1^{m_{l-1}}g_2^{m_l}\right)[f]\right]\\
		=&\Psi[\pi(g)[f]].
	\end{align*}
	\label{proof2}
\end{proof}

\begin{theorem}
	When $\mathcal{G}=SO(3)$, if Eqn. (\ref{equivariance}) or (\ref{base}) is satisfied for the approximated generator $g_1=Z(1)$ and $g_2 = Y(1)$, where
	\begin{scriptsize}
		\begin{align*}
			Z(\alpha)=\left[
			\begin{array}{ccc}
				\cos\alpha & -\sin\alpha & 0\\
				\sin\alpha & \cos\alpha & 0\\
				0 & 0 & 1\\
			\end{array}
			\right],\, Y(\beta)=\left[
			\begin{array}{ccc}
				\cos\beta & 0  & \sin\beta\\
				0  & 1  & 0\\
				-\sin\beta & 0 & \cos\beta
			\end{array}
			\right],
		\end{align*}
	\end{scriptsize}
	then $\forall g\in \mathcal{G}$, Eqn. (\ref{equivariance}) is satisfied.
	\label{thm3}
\end{theorem}
\begin{proof}
	$\forall \alpha\in [0,2\pi)$, there exists a sequence of positive integers $\{n_k\in \mathbb{Z}^{+}\}$, such that 
	\begin{equation*}
		\lim_{k\to +\infty} Z\left(n_k\right) = Z(\alpha).
	\end{equation*}
	As a result, since Eqn. (\ref{equivariance}) (or Eqn.(\ref{base})) is satisfied for $g_1$, i.e.,
	\begin{equation*}
		\pi'\left(g_1\right)[\Psi[f]]=\Psi\left[\pi(g_1)[f]\right],
	\end{equation*}
	and $Z\left(n_k\right)=g_1^{n_k}$, we have that
	\begin{align*}
		\pi'\left(Z\left(n_k\right)\right)[\Psi[f]]=&\pi'\left(g_1^{n_k}\right)[\Psi[f]]\\
		=&\pi'\left(g_1\right)^{n_k}\left[\Psi[f]\right]\\
		=&\pi'\left(g_1\right)^{n_k-1}\pi'\left(g_1\right)\left[\Psi[f]\right]\\
		=&\pi'\left(g_1\right)^{n_k-1}\Psi\left[\pi(g_1)[f]\right],\\
		=&\Psi\left[\pi\left(g_1^{n_k}\right)[f]\right]\\
		=&\Psi\left[\pi\left(Z(n_k)\right)[f]\right].
	\end{align*}
	As a result, we have that
	\begin{align*}
		\pi'\left(Z(\alpha)\right)[\Psi[f]]=&\lim_{k\to +\infty}\pi'\left(Z(n_k)\right)[\Psi[f]]\\
		=&\lim_{k\to +\infty}\Psi[\pi\left(Z(n_k)\right)[f]]\\
		=&\Psi\left[\pi\left(Z(\alpha)\right)[f]\right].
	\end{align*}
	Analogously, we can prove that $\forall \beta\in [0,\pi] $, 
	\begin{equation*}
		\pi'\left(Y(\beta)\right)[\Psi[f]]=\Psi\left[\pi\left(Y(\beta)\right)[f]\right].	
	\end{equation*}
	Actually, any $g\in SO(3)$ can be decomposed using the ZYZ Euler parameterization, i.e.,
	\begin{equation*}
		g = Z\left(\alpha_g\right)Y\left(\beta_g\right)Z\left(\gamma_g\right),
	\end{equation*}
	where $\alpha_g, \gamma_g\in [0,2\pi)$ and $\beta_g \in [0,\pi]$. Consequently, it is easy to deduce that
	\begin{align*}
		\pi'(g)[\Psi[f]]=&\pi'(Z\left(\alpha_g)\right)\pi'\left(Y(\beta_g)\right)\pi'\left(Z(\gamma_g\right))[\Psi[f]]\\
		=&\Psi\left[\pi\left(Z\left(\alpha_g\right)\right)\pi\left(Y\left(\beta_g\right)\right)\pi\left(Z\left(\gamma_g\right)\right)[f]\right]\\
		=&\Psi\left[\pi\left(Z\left(\alpha_g\right)Y\left(\beta_g\right)Z\left(\gamma_g\right)\right)[f]\right]\\
		=&\Psi\left[\pi(g)[f]\right].
	\end{align*}
\end{proof}

\section{The Schema and the Generators of Different Groups\label{appc}}
See Figure \ref{polyhedron} and Table \ref{generators}.
\begin{figure*}[t]
	\centering
	\includegraphics[width=0.9\textwidth]{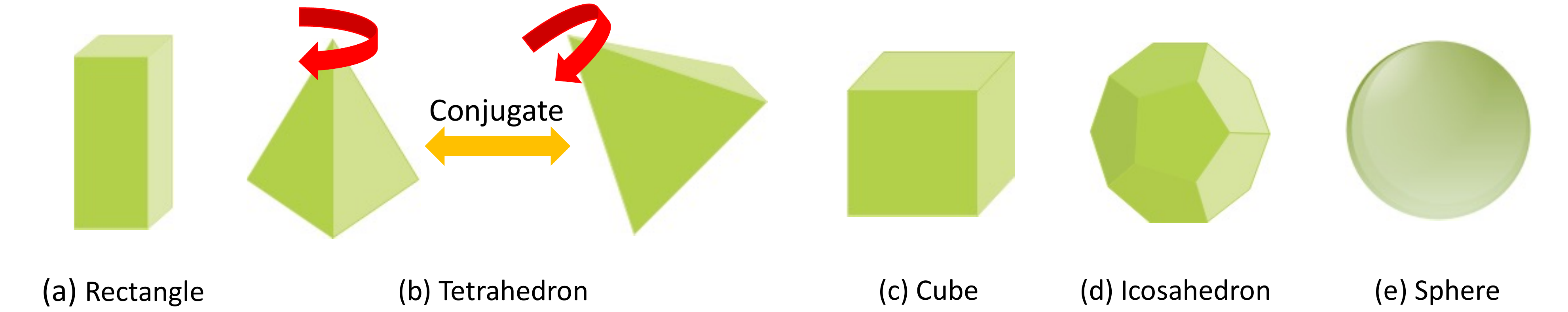} 
	\caption{3D rotation groups $\mathcal{V},\mathcal{T},\mathcal{O},\mathcal{I}$ and $SO(3)$ are composed of all the rotational symmetries of (a) rectangle, (b) tetrahedron, (c) cube, (d) icosahedron, and (e) sphere, respectively. (b) Two tetrahedron groups are conjugate, as they are composed of the rotational symmetries of a tetrahedron and a rotated tetrahedron, respectively.}
	\label{polyhedron}
\end{figure*}

\begin{table}[t]
	\caption{The generators and sizes of discrete groups $\mathcal{V},\mathcal{T},\mathcal{O}$ and $\mathcal{I}$, where $\phi=(1+\sqrt{5})/2$.}
	\centering
	\scriptsize
	\begin{tabular}{ccc}
		\toprule
		Group   & Generators & Size \\
		\midrule
		$\mathcal{V}$  &$\left[
		\begin{array}{p{0.5cm}<{\centering}p{0.5cm}<{\centering}p{0.5cm}<{\centering}}
			-1 & 0 & 0\\
			0 & -1 & 0\\
			0 & 0 & 1\\
		\end{array}	
		\right],
		\left[
		\begin{array}{p{0.5cm}<{\centering}p{0.5cm}<{\centering}p{0.5cm}<{\centering}}
			1 & 0 & 0\\
			0 & -1 & 0\\
			0 & 0 & -1\\
		\end{array}	
		\right]$ &  4\\
		\hline
		$\mathcal{T}$ &	$ \left[
		\begin{array}{p{0.5cm}<{\centering}p{0.5cm}<{\centering}p{0.5cm}<{\centering}}
			0 & 0 & 1\\
			1 & 0 & 0\\
			0 & 1 & 0\\
		\end{array}	
		\right],
		\left[
		\begin{array}{p{0.5cm}<{\centering}p{0.5cm}<{\centering}p{0.5cm}<{\centering}}
			-1 & 0 & 0\\
			0 & -1 & 0\\
			0 & 0 & 1\\ 
		\end{array}	
		\right]$ & 12\\
		\hline
		$\mathcal{O}$ &$ \left[
		\begin{array}{p{0.5cm}<{\centering}p{0.5cm}<{\centering}p{0.5cm}<{\centering}}
			0 & -1 & 0\\ 
			1 & 0 & 0\\
			0 & 0 & 1\\
		\end{array}
		\right],
		\left[
		\begin{array}{p{0.5cm}<{\centering}p{0.5cm}<{\centering}p{0.5cm}<{\centering}}
			0 & 0 & 1\\
			0 & 1 & 0\\
			-1 & 0 & 0\\ 
		\end{array}	
		\right]$  & 24 \\
		\hline
		$\mathcal{I}$ &	$ \left[
		\begin{array}{p{0.5cm}<{\centering}p{0.5cm}<{\centering}p{0.5cm}<{\centering}}
			-1 & 0 & 0\\
			0 & -1 & 0\\
			0 & 0 & 1\\
		\end{array}	
		\right],
		\left[
		\begin{array}{p{0.5cm}<{\centering}p{0.5cm}<{\centering}p{0.5cm}<{\centering}}
			$\frac{1-\phi}{2}$ & $\frac{\phi}{2}$ & $-\frac{1}{2}$\\
			$-\frac{\phi}{2}$ & $-\frac{1}{2}$ & $\frac{1-\phi}{2}$\\
			$-\frac{1}{2}$ & $\frac{\phi-1}{2}$ & $\frac{\phi}{2}$\\ 
		\end{array}	
		\right]$  & 60 \\
		\bottomrule
	\end{tabular}
	\label{generators}
\end{table}

\section{Implementation\label{appd}}
\subsection{Finite Difference (FD)\label{appd1}}
The volumetric data/feature $\bm{F}$ can be viewed as a sample from a continuous function $f$ defined on $\mathbb{R}^3$. The grid size along $x_1$- , $x_2$- and $x_3$-axis are $h_1,h_2$ and $h_3$, respectively, and the grid diameter $\rho=\sqrt{h_1^2+h_2^2+h_3^2}$. When using FD for discretization, we have that $\forall i\in S_x$,
\begin{equation*}
	\partial_i [f]= u_i \ast \bm{F} + O\left(\rho^2\right),
\end{equation*}
where $S_x=\{x_1,x_2,x_3,x_1^2,x_2^2,x_3^2,x_1x_2,x_1x_3,x_2x_3\}$, and
\begin{equation*}
	\left\{
	\renewcommand{\arraystretch}{1.6}
	\begin{array}{l}
		u_{x_1}(:,0,0)=\frac{1}{h_1}\left[
		\begin{array}{ccc}
			-1/2 & 0 & 1/2\\
		\end{array}	
		\right],\\
		u_{x_2}(0,:,0)=\frac{1}{h_2}\left[
		\begin{array}{ccc}
			-1/2 & 0 & 1/2\\
		\end{array}	
		\right],\\
		u_{x_3}(0,0,:)=\frac{1}{h_3}\left[
		\begin{array}{ccc}
			-1/2 & 0 & 1/2\\
		\end{array}	
		\right],\\
		u_{x_1^2}(:,0,0)=\frac{1}{h_1^2}\left[
		\begin{array}{ccc}
			1 & -2 & 1\\
		\end{array}	
		\right],\\
		u_{x_2^2}(0,:,0)=\frac{1}{h_2^2}\left[
		\begin{array}{ccc}
			1 & -2 & 1\\
		\end{array}	
		\right],\\
		u_{x_3^2}(0,:,0)=\frac{1}{h_3^2}\left[
		\begin{array}{ccc}
			1 & -2 & 1\\
		\end{array}	
		\right],\\
		u_{x_1x_2}(:,:,0)=\frac{1}{h_1h_2}\left[
		\begin{array}{ccc}
			-1/4 & 0 & 1/4\\
			0 & 0 & 0\\
			1/4 & 0 & -1/4\\ 
		\end{array}	
		\right].\\
		u_{x_1x_3}(:,0,:)=\frac{1}{h_1h_3}\left[
		\begin{array}{ccc}
			-1/4 & 0 & 1/4\\
			0 & 0 & 0\\
			1/4 & 0 & -1/4\\ 
		\end{array}	
		\right],\\
		u_{x_2x_3}(0,:,:)=\frac{1}{h_2h_3}\left[
		\begin{array}{ccc}
			-1/4 & 0 & 1/4\\
			0 & 0 & 0\\
			1/4 & 0 & -1/4\\ 
		\end{array}	
		\right].\\
	\end{array}
	\right.
\end{equation*}
We only show some elements of each convolution filter for ease of presentation, and other elements are all zeros. 

Particularly, for regular 3D grids, if we discretize a $\mathcal{G}$-steerable mapping $\Psi$ using FD, where $\mathcal{G}$ includes the cubic group $\mathcal{O}$ as a subset, i.e., $\mathcal{O}\leq \mathcal{G}$, then discretized $\hat \Psi$ can preserve exact equivariance over $\mathcal{O}$ in the discrete domain. Formally, we have the following theorem.
\begin{theorem}
	If $\Psi$ is equivariant over $\mathcal{G}$, where $\mathcal{O}\leq \mathcal{G}$, then $\hat\Psi$ is exactly equivariant over $\mathcal{O}$, i.e., $\forall g\in \mathcal{O}$, we have
	\begin{equation}
		\pi'(g)\left[\hat\Psi[\bm{F}]\right]=\hat\Psi\left[\pi(g)\bm{F}\right],\label{dequivariance}
	\end{equation}
	where the volumetric data/feature $\bm{F}$ is defined on regular 3D grids, $h_1=h_2=h_3$, and
	\begin{equation*}
		\left\{
		\begin{array}{l}
			[\pi(g)\bm{F}](x) = \rho(g)\bm{F}\left(g^{-1}x\right),\\
			\left[\pi'(g)\bm{F}\right](x) = \rho'(g)\bm{F}\left(g^{-1}x\right).\\
		\end{array}
		\right.
		\label{rho}
	\end{equation*}
\label{thd1}
\end{theorem}
\begin{proof}
	The proof is mostly following the proof of Theorem \ref{thm}, and most symbols and procedures are similar for ease of understanding. Specifically, Eqn. (\ref{dequivariance}) can be rewritten as $\forall g\in \mathcal{O}$,
	\begin{align}
		&\rho'(g)\left(A_0+A_1u_{x_1}+A_2u_{x_2}+A_3u_{x_3}+A_{11}u_{x_1^2}\right.\notag\\
		&\left.+A_{12}u_{x_1x_2}+A_{13}u_{x_1x_3}+A_{22}u_{x_2^2}+A_{23}u_{x_2x_3}\right.\notag\\
		&\left.+A_{33}u_{x_3^2}\right)\ast\bm{F}\left(g^{-1}x\right)\notag\\
		=&\left(A_0+A_1u_{x_1}+A_2u_{x_2}+A_3u_{x_3}+A_{11}u_{x_1^2}\right.\notag\\
		&\left.+A_{12}u_{x_1x_2}+A_{13}u_{x_1x_3}+A_{22}u_{x_2^2}+A_{23}u_{x_2x_3}\right.\notag\\
		&\left.+A_{33}u_{x_3^2}\right)\ast\left[\rho(g)\bm{F}\left(g^{-1}x\right)\right].
		\label{dextend}
	\end{align}	
	We denote that $\forall i = 1,2,\cdots,K,$
	\begin{align*}
		\hat\nabla^T \left[\bm{F}_i\right]=& \left[u_{x_1}*\bm{F}_i,u_{x_2}\ast\bm{F}_i,u_{x_3}\ast\bm{F}_i\right],\\
		\hat\nabla^2 \left[\bm{F}_i\right]=&\left[
		\begin{array}{ccc}
			u_{x_1^2}\ast\bm{F}_i & u_{x_1x_2}\ast\bm{F}_i& u_{x_1x_3}\ast\bm{F}_i\\
			u_{x_1x_2}\ast\bm{F}_i & u_{x_2^2}\ast\bm{F}_i & u_{x_2x_3}\ast\bm{F}_i\\	 
			u_{x_1x_3}\ast\bm{F}_i & u_{x_2x_3}\ast\bm{F}_i & u_{x_3^2}\ast\bm{F}_i\\	 
		\end{array}
		\right].
	\end{align*}
	Then the left hand side of Eqn. (\ref{dextend}) can be written as
	\begin{equation}
		\rho'(g)\left(B_0\bm{F}\left(g^{-1}x\right)+B_1 \hat V_1+\bar{B}_2\hat{\bar{V}}_2\right)\label{lfh},
	\end{equation}
	where
	\begin{small}
		\begin{align*}
			\hat V_1=&vec\left[
			\begin{array}{c}
				\hat\nabla^T \left[\bm{F}_1\right]\left(g^{-1}x\right)\\
				\hat\nabla^T \left[\bm{F}_2\right]\left(g^{-1}x\right)\\
				\cdots\\
				\hat\nabla^T\left [\bm{F}_K\right]\left(g^{-1}x\right)\\
			\end{array}
			\right]\\
			=&
			vec\left[u_{x_1}\ast\bm{F}\left(g^{-1}x\right),u_{x_2}\ast\bm{F}\left(g^{-1}x\right),u_{x_3}\ast\bm{F}\left(g^{-1}x\right)\right],\\
		\end{align*}
	\end{small}
	and
	\begin{scriptsize}
		\begin{align*}
			\hat{\bar{V}}_2 = &vec\left[
			\begin{array}{c}
				vec^T\left(\hat\nabla^2 \left[\bm{F}_1\right]\left(g^{-1}x\right)\right)\\
				vec^T\left(\hat\nabla^2 \left[\bm{F}_2\right]\left(g^{-1}x\right)\right)\\
				\cdots\\
				vec^T\left(\hat\nabla^2 \left[\bm{F}_K\right]\left(g^{-1}x\right)\right)\\
			\end{array}
			\right]\\
			=&vec
			\left[
			u_{x_1^2}\ast\bm{F}\left(g^{-1}x\right),u_{x_1x_2}\ast\bm{F}\left(g^{-1}x\right),u_{x_1x_3}\ast\bm{F}\left(g^{-1}x\right),\right.\\
			&\left.	u_{x_1x_2}\ast\bm{F}\left(g^{-1}x\right),u_{x_2^2}\ast\bm{F}\left(g^{-1}x\right),u_{x_2x_3}\ast\bm{F}\left(g^{-1}x\right),\right.\\
			&\left.u_{x_1x_3}\ast\bm{F}\left(g^{-1}x\right),u_{x_2x_3}\ast\bm{F}\left(g^{-1}x\right),u_{x_3^2}\ast\bm{F}\left(g^{-1}x\right)\right].
		\end{align*}
	\end{scriptsize}
	We further denote that
	\begin{small}
		\begin{align*}
			\hat V_2 =&vec
			\left[
			u_{x_1^2}\ast\bm{F}\left(g^{-1}x\right),u_{x_1x_2}\ast\bm{F}\left(g^{-1}x\right),u_{x_1x_3}\ast\bm{F}\left(g^{-1}x\right),\right.\\
			&\left.u_{x_2^2}\ast\bm{F}\left(g^{-1}x\right),u_{x_2x_3}\ast\bm{F}\left(g^{-1}x\right),u_{x_3^2}\ast\bm{F}\left(g^{-1}x\right)\right],
		\end{align*}
	\end{small}
	and we have 
	\begin{equation*}
		\hat{\bar{V}}_2 = \left(P^{\dag}\otimes I_K\right)\hat V_2,
	\end{equation*}
	As a result, the left hand size of Eqn. (\ref{dextend}) can be written as
	\begin{scriptsize}
		\begin{align*}
			&\rho'(g)\left(B_0\bm{F}\left(g^{-1}x\right)+B_1 \hat V_1+\bar{B}_2\hat{\bar{V}}_2\right)\\
			=&\rho'(g)\left(B_0\bm{F}\left(g^{-1}x\right)+B_1 \hat V_1+B_2\left(P\otimes I_K\right)\left(P^{\dag}\otimes I_K\right)\hat V_2\right)\\
			=&\rho'(g)\left(B_0\bm{F}\left(g^{-1}x\right)+B_1 \hat V_1+B_2\hat V_2\right).\\
		\end{align*}
	\end{scriptsize}
	Also, the right hand side of (\ref{dextend}) can be written as
	\begin{scriptsize}
		\begin{align*}
			&\left(A_0+A_1u_{x_1}+A_2u_{x_2}+A_3u_{x_3}+A_{11}u_{x_1^2}\right.\\
			&\left.+A_{12}u_{x_1x_2}+A_{13}u_{x_1x_3}+A_{22}u_{x_2^2}+A_{23}u_{x_2x_3}\right.\notag\\
			&\left.+A_{33}u_{x_3^2}\right)\ast\left[\rho(g)\bm{F}\left(g^{-1}x\right)\right]\\
			=	&\left(A_0\rho(g)+A_1\rho(g)u_{x_1}+A_2\rho(g)u_{x_2}+A_3\rho(g)u_{x_3}\right.\\
			&\left.+A_{11}\rho(g)u_{x_1^2}+A_{12}\rho(g)u_{x_1x_2}+A_{13}\rho(g)u_{x_1x_3}\right.\\
			&\left.+A_{22}\rho(g)u_{x_2^2}+A_{23}\rho(g)u_{x_2x_3}+A_{33}\rho(g)u_{x_3^2}\right)\ast\left[\bm{F}\left(g^{-1}x\right)\right]\\
			=&B_0\rho(g)\bm{F}\left(g^{-1}x\right)+B_1\left(I_3\otimes \rho(g)\right) vec\left[
			\begin{array}{c}
				\hat\nabla^T \left[\bm{F}_1\left(g^{-1}x\right)\right]\\
				\hat\nabla^T \left[\bm{F}_2\left(g^{-1}x\right)\right]\\
				\cdots\\
				\hat\nabla^T \left[\bm{F}_K\left(g^{-1}x\right)\right]\\
			\end{array}
			\right]\\
			&+\bar{B}_2\left(I_9\otimes\rho(g)\right)vec\left[
			\begin{array}{c}
				vec^T\left(\hat\nabla^2 \left[\bm{F}_1\left(g^{-1}x\right)\right]\right)\\
				vec^T\left(\hat\nabla^2 \left[\bm{F}_2\left(g^{-1}x\right)\right]\right)\\
				\cdots\\
				vec^T\left(\hat\nabla^2 \left[\bm{F}_K\left(g^{-1}x\right)\right]\right)\\
			\end{array}
			\right]\\
		\end{align*}
	\end{scriptsize}
	As for volumetric input $\bm{F}_i$ defined on regular 3D grids with $h_1=h_2=h_3$, it can be verified that $\forall g\in \mathcal{O}$,
	\begin{align*}
		\hat\nabla^T \left[\bm{F}_i\left(g^{-1}x\right)\right]=&	\hat\nabla^T \left[\bm{F}_i\right]\left(g^{-1}x\right)g^{-1},\\
		\hat\nabla^2 \left[\bm{F}_i\left(g^{-1}x\right)\right] =& g\hat\nabla^2 \left[\bm{F}_i\right]\left(g^{-1}x\right)g^{-1}.
	\end{align*}
	Then similar to that in the proof of Theorem \ref{thm}, the previous equation can be further written as
	\begin{scriptsize}
		\begin{align*}
			&B_0\rho(g)\bm{F}\left(g^{-1}x\right)+B_1\left(I_3\otimes \rho(g)\right) vec\left[
			\begin{array}{c}
				\hat\nabla^T \left[\bm{F}_1\right]\left(g^{-1}x\right)g^{-1}\\
				\hat\nabla^T \left[\bm{F}_2\right]\left(g^{-1}x\right)g^{-1}\\
				\cdots\\
				\hat\nabla^T \left[\bm{F}_K\right]\left(g^{-1}x\right)g^{-1}\\
			\end{array}
			\right]\\
			&+\bar{B}_2\left(I_9\otimes\rho(g)\right)vec\left[
			\begin{array}{c}
				vec^T\left(g\hat\nabla^2 \left[\bm{F}_1\right]\left(g^{-1}x\right)g^{-1}\right)\\
				vec^T\left(g\hat\nabla^2 \left[\bm{F}_2\right]\left(g^{-1}x\right)g^{-1}\right)\\
				\cdots\\
				vec^T\left(g\hat\nabla^2 \left[\bm{F}_K\right]\left(g^{-1}x\right)g^{-1}\right)\\
			\end{array}
			\right]\\
			=	&B_0\rho(g)\bm{F}\left(g^{-1}x\right)+B_1\left(g\otimes \rho(g)\right)\hat V_1\\
			&+B_2\left(P\left(g\otimes g\right)P^{\dag}\otimes\rho(g)\right)\hat V_2.\\
		\end{align*}
	\end{scriptsize}
	As a result, Eqn. (\ref{dequivariance}) is equivalent to that $\forall g\in \mathcal{O}$,
		\begin{align*}
			&\rho'(g)B_0\bm{F}\left(g^{-1}x\right)+\rho'(g)B_1 \hat V_1+\rho'(g)B_2\hat V_2\\
			=&B_0\rho(g)\bm{F}\left(g^{-1}x\right)+B_1\left(g\otimes \rho(g)\right)\hat V_1\\
			&+B_2\left(P\left(g\otimes g\right)P^{\dag}\otimes\rho(g)\right)\hat V_2.
		\end{align*}
	Since $\Psi$ is equivariant over $\mathcal{O}$, Eqn. (\ref{base}) is satisfied for any $g$ in $\mathcal{O}$. Finally, Eqn. (\ref{dequivariance}) is satisfied.
	\label{proof3}
\end{proof}

\subsection{Nonlinearities and Batch Normalization (BN)\label{appd2}}
\textbf{Scale-like Nonlinearities}\quad As irreducible representations (irreps) we employ are all unitary representations which preserve the norm of features, i.e., they satisfy that $\|\rho(g)f(x)\|_2=\|f(x)\|_2$. As a result, nonlinearities which solely act on the norm of feature vectors but preserve their orientation are admissible. They can in general be decomposed in $\sigma_{\text{norm}}:f(x)\mapsto \eta(\|f(x)\|_2)\frac{f(x)}{\|f(x)\|_2}$, where $\eta$ is a nonlinear function. For instance, Norm-ReLUs are defined by using $\eta(\|f(x)\|_2)=ReLU(\|f(x)\|_2-b)$ where $b$ is a learnable bias. Gated nonlinearity scales the norm of a feature field by learned sigmoid gates $\frac{1}{1+e^{-s(x)}}$, parameterized by  a scalar feature field $s$. In all, existing nonlinearities act by scaling the feature vectors, where the scales are acquired by different methods that do not disturb equivariance. We employ gated nonlinearity for implementation as many works have shown that it works in practice better than other options \citep{weiler20183d,weiler2019general}.

\noindent\textbf{Scale-BN}\quad Non-trivial irreducible feature fields are normalized with the average of their norms:
\begin{equation*}
	f(x)\mapsto f(x)\left(\frac{1}{|\mathcal{B}|V}\sum_{j\in\mathcal{B}}\int \|f(x)\|_2^2dx\right)^{-1/2},
\end{equation*}
where $\mathcal{B}$ is the batch size and $V$ is the size of domain. As a result, this kind of BN can only scale the feature vectors and cannot correct activations to zero-mean.

\subsection{Composition of Basic Feature Fields\label{appd3}}
We can easily employ one single basic feature field per layer, such as the regular, quotient, and irreducible feature field. Actually, we can put multiple basic feature fields together and acquire much more general feature fields, and the constraints for coefficients can be solved efficiently. Formally, we have the following theorem.

\begin{theorem}
	If $\rho(g)=\textcircled{+}_{i=1}^t \rho_i(g)$ and $\rho'(g)=\textcircled{+}_{i=1}^{s} \rho'_i(g)$, where $\rho_i(g)$ are $\rho'_i(g)$ are both basic group representations, and
	\begin{equation*}
		\textcircled{+}_{i=1}^t \rho_i(g)=\left[
		\begin{array}{ccc}
			\rho_1(g) & &  \\
			&  \ddots & \\
			& & \rho_{t}(g)
		\end{array}
		\right],
	\end{equation*}
	then Eqn. (\ref{base}) is equivalent to 
	\begin{equation*}
		\left\{
		\renewcommand{\arraystretch}{1.3}
		\begin{array}{l}
			\rho_i'(g)C_0^{ij}=C_0^{ij}\rho_j(g),\\
			\rho_i'(g)C_1^{ij}=C_1^{ij}\left(g\otimes \rho_j(g)\right),\\
			\rho_i'(g)C_2^{ij}=C_2^{ij}\left(P\left(g\otimes g\right)P^{\dag}\otimes \rho_j(g)\right),
		\end{array}
		\right.
		\label{element}
	\end{equation*}
	where
	\begin{equation}
		\left\{
		\renewcommand{\arraystretch}{1.3}
		\begin{array}{l}
			C^{ij}_0 = A^{ij}_0,\\
			C_1^{ij}=\left[A^{ij}_1,A^{ij}_2,A^{ij}_3\right],\\
			C_2^{ij}=\left[A^{ij}_{11},A^{ij}_{12},A^{ij}_{13},A^{ij}_{22},A^{ij}_{23},A^{ij}_{33}\right],
		\end{array}
		\right.
		\label{C}
	\end{equation}
	and $A^{ij}_k$ is the submatrix of $A_k$, i.e.,
	\begin{equation*}
		A_k=\left[
		\begin{array}{ccc}
			A_k^{11} & \cdots & A_k^{1t}\\
			\vdots & & \vdots\\
			A_k^{s1} & \cdots & A_k^{st}\\
		\end{array}
		\right].
		\label{A}
	\end{equation*}
\end{theorem}
\begin{proof}
	The constraint for $B_0=A_0$ in (\ref{base}) can be written as
	\begin{small}
		\begin{align*}
			&\left[
			\begin{array}{ccc}
				\rho'_1(g) & &  \\
				&  \ddots & \\
				& & \rho'_{s}(g)
			\end{array}
			\right]
			\left[
			\begin{array}{ccc}
				A_0^{11} & \cdots & A_0^{1t} \\
				\vdots&   &  \vdots\\
				A_0^{s1}& \cdots& A_0^{st}
			\end{array}
			\right]\\
			=&
			\left[
			\begin{array}{ccc}
				A_0^{11} & \cdots & A_0^{1t} \\
				\vdots&   &  \vdots\\
				A_0^{s1}& \cdots& A_0^{st}
			\end{array}
			\right]
			\left[
			\begin{array}{ccc}
				\rho_1(g) & &  \\
				&  \ddots & \\
				& & \rho_{t}(g)
			\end{array}
			\right],
		\end{align*}
	\end{small}
	i.e., $\forall i=1,2,\cdots,s$ and $\forall j=1,2,\cdots,t,$
	\begin{equation*}
		\rho_i'(g)C_0^{ij}=C_0^{ij}\rho_j(g),
	\end{equation*}
	where $C_0^{ij}=A_0^{ij}$. As for the constraint for $B_1=\left[A_1,A_2,A_3\right]$, we have
	\begin{align*}
		\rho'(g)\left[A_1,A_2,A_3\right]=\left[A_1,A_2,A_3\right]\left(g\otimes \rho(g)\right),
	\end{align*}
	i.e.,
	\begin{align*}
		&\rho'(g)[A_1,A_2,A_3]\\
		=&[A_1,A_2,A_3]
		\left[
		\begin{array}{ccc}
			g_{11}\rho(g) & g_{12}\rho(g) & g_{13}\rho(g)\\
			g_{21}\rho(g) & g_{22}\rho(g) & g_{23}\rho(g)\\
			g_{31}\rho(g) & g_{32}\rho(g) & g_{33}\rho(g)\\
		\end{array}	
		\right],
	\end{align*}
	i.e.,
	\begin{align}
		&[\rho'(g)A_1,\rho'(g)A_2,\rho'(g)A_3]\notag\\
		=&\left[g_{11}A_1\rho(g) +g_{21}A_2\rho(g)+g_{31}A_3\rho(g),\right.\notag\\
		&\left.g_{12}A_1\rho(g)+g_{22}A_2\rho(g)+g_{32}A_3\rho(g),\right.\notag\\
		&\left.g_{13}A_1\rho(g)+g_{23}A_2\rho(g)+g_{33}A_3\rho(g)\right].
		\label{element}
	\end{align}
	
	Eqn. (\ref{element}) can be further expanded as
	\begin{align*}
		&\left[\left(\rho_i'(g)A_1^{ij}\right)_{ij},\left(\rho_i'(g)A_2^{ij}\right)_{ij},\left(\rho_i'(g)A_3^{ij}\right)_{ij}\right]\\
		=&\left[\left(g_{11}A_1^{ij}\rho_j(g) +g_{21}A_2^{ij}\rho_j(g)+g_{31}A_3^{ij}\rho_j(g)\right)_{ij},\right.\notag\\
		&\left.\left(g_{12}A_1^{ij}\rho_j(g)+g_{22}A_2^{ij}\rho_j(g)+g_{32}A_3^{ij}\rho_j(g)\right)_{ij},\right.\notag\\
		&\left.\left(g_{13}A_1^{ij}\rho_j(g)+g_{23}A_2^{ij}\rho_j(g)+g_{33}A_3^{ij}\rho_j(g)\right)_{ij}\right],
	\end{align*}
	i.e., $\forall i=1,2,\cdots,s$ and $\forall j=1,2,\cdots,t,$
	\begin{align*}
		&\left[\rho_i'(g)A_1^{ij},\rho_i'(g)A_2^{ij},\rho_i'(g)A_3^{ij}\right]\\
		=	&\left[g_{11}A_1^{ij}\rho_j(g) +g_{21}A_2^{ij}\rho_j(g)+g_{31}A_3^{ij}\rho_j(g),\right.\notag\\
		&\left.g_{12}A_1^{ij}\rho_j(g)+g_{22}A_2^{ij}\rho_j(g)+g_{32}A_3^{ij}\rho_j(g),\right.\notag\\
		&\left.g_{13}A_1^{ij}\rho_j(g)+g_{23}A_2^{ij}\rho_j(g)+g_{33}A_3^{ij}\rho_j(g)\right],
	\end{align*}
	and this can be reversely written as 
	\begin{equation*}
		\rho_i'(g)\left[A^{ij}_1,A^{ij}_2,A^{ij}_3\right]=\left[A^{ij}_1,A^{ij}_2,A^{ij}_3\right](g\otimes \rho_j(g)),
	\end{equation*}
	i.e.,
	\begin{equation*}
		\rho_i'(g)C_1^{ij}=C_1^{ij}(g\otimes \rho_j(g)),
	\end{equation*}
	where $C_1^{ij}=\left[A^{ij}_1,A^{ij}_2,A^{ij}_3\right]$.
	
	Analogously, the constraints for $B_2$ is equivalent to that 
	\begin{equation*}
		\rho_i'(g)C_2^{ij}=C_2^{ij}\left(P(g\otimes g)P^{\dag}\otimes \rho_j(g)\right),
	\end{equation*}
	where $C_2^{ij}=\left[A^{ij}_{11},A^{ij}_{12},A^{ij}_{13},A^{ij}_{22},A^{ij}_{23},A^{ij}_{33}\right]$.  
\end{proof}
In this way, we can firstly compute $C^{ij}_0,C^{ij}_1$ and $C^{ij}_2$ according to the basic group representations $\rho_i'(g)$ and $\rho_j(g)$, then $A_k^{ij}$ are obtained according to Eqn. (\ref{C}). Finally, the coefficients $A_k$ are obtained by placing $A_k^{ij}$ properly.

\section{Experimental Details\label{appe}}
We provide the experimental details in this section. Our experiments are implemented using Pytorch, and each one is run using one single Tesla-V100 GPU.
\subsection{3D Tetris (Testing 3D equivariance)\label{appe1}}
This dataset contains 8 shapes, and we convert each Tetris block into $40\times40\times40$ voxels for preprocessing. The $\mathcal{O}$-steerable models consist of $3$ convolutional layers. The first two layers are both composed of $10$ basic features, and the last layer is composed of $64$ scalar features. The architecture of the $SO(3)$-steerable model is in Table \ref{SO3_tetris}, which shows the numbers and sizes of feature fields. We use the average pooling for downsampling and a global average pooling after the final convolution layers to yield an invariant representation. Our models are trained using the Adam algorithm \cite{kingma2014adam}. We use the generalized He’s weight initialization scheme introduced in \cite{weiler2018learning} for the convolutional layers and the Xavier optimizer \cite{glorot2010understanding} for the fully-connected (FC) layers. We train for $200$ epochs with an initial learning rate of $0.01$ and an exponential decay of $0.98$ after $50$ epochs. 

For equivariance error analysis, we randomly sample $n=100$ rotations $g_i$ and shapes $\bm{I}_i$, and compute equivariance error $\frac{1}{n}\sum_{i=1}^n\|\Phi\left[\pi(g_i)[\bm{I}_i]\right]-\Phi[\bm{I}_i]\|_2/\|\Phi[\bm{I}_i]\|_2$, where $\Phi$ is a composition of $SO(3)$-steerable layers following by an invariant layer with trivial output features. Each intermediate feature fieds are $\textcircled{+}_{i=0}^2 D^l_i(g)$, and the models are tested after weights being initialized without training. The standard deviation $\sigma$ for Gaussian discretization is taken as half the radius of the kernel size. That is, $\sigma=0.5$ for the kernel size of $3\times 3\times 3$ ($k=3$), and $\sigma=1.0$ for $5\times 5\times 5$ ($k=5$). As shown in Figure \ref{error}, Gaussian discretization results in a much lower error for $SO(3)$-equivariance. 


\begin{table}[t]
	\caption{The architecture of the $SO(3)$-steerable model with irreducible features.}
	\centering
	\begin{tabular}{lcccc}
		\toprule
		&  $D^0(\rho)$ & $D^1(\rho)$  & $D^2(\rho)$ & Size \\
		\midrule
		Input  & 1 &0  & 0& $40^3$\\
		Layer1 & 4 & 4 & 4 & $40^3$ \\
		Layer2 & 16 & 16 & 16 & $20^3$ \\
		Layer3 & 32 & 16 & 16 & $10^3$ \\
		Layer4 & 128 &0 &0  & $10^3$ \\
		Output & 8 & 0 & 0 & $1$\\
		\bottomrule
	\end{tabular}
	\label{SO3_tetris}
\end{table}

\begin{figure}[t]
	\centering
	\includegraphics[scale=0.55]{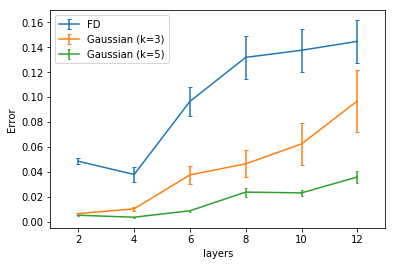}
	\caption{The $SO(3)$-equivariance errors from FD and Gaussian discretizations, respectively.}
	\label{error}
\end{figure}

\begin{table}[t]
	\caption{The results when we use different combinations of PDOs. $I,\nabla$ and $\nabla^2$ denote the identity operator, the first and the second order PDOs, respectively.}
	\centering
	\begin{tabular}{lc}
		\toprule
		Convolution kernel & Score \\
		\midrule
		$I+\nabla$ & 56.2 \\
		$I+\nabla^2$ ($\mathcal{V}$-quotient features)& 55.5 \\
		$I+\nabla^2$ (regular features)& 57.3 \\
		\hline 
		$I+\nabla+\nabla^2$&  \textbf{58.6}\\
		\bottomrule
	\end{tabular}
	\label{ablation}
\end{table}




\subsection{SHREC'17 Retrieval Challenge\label{appe2}}
We choose models and hyperparameters with the lowest validation error during training. The architectures of steerable models for the SHREC'17 retrieval task are shown in Table \ref{V_shrec}-\ref{S4_no_p2_shrec}. Our models are trained using the Adam optimizer \cite{kingma2014adam}. We train for $2000$ epochs with a batch size of $32$. The initial learning rate is set to $0.01$ and is divided by $10$ at $700$ and $1,400$ epochs. The filters from Gaussian discretization are of the size $5\times 5\times 5$, and $\sigma=1.0$.

\noindent\textbf{Ablation Study}\quad We conduct an ablation study to justify the choice of PDOs in Eqn. (\ref{Psi}). We choose several combinations of PDOs and report the final scores based on the $\mathcal{O}$-steerable models with regular features. For all settings, we use a model with about 0.15M parameters for evaluation on the SHREC'17 retrieval task.  Results for this ablation study are listed in Table \ref{ablation}. Our choice of PDO combinations in Eqn. (\ref{Psi}) achieves the best performance, indicating the necessity of employing all the PDOs up to the second-order. 

Interestingly, we observe that when we do not use the first-order PDOs, the model with regular features still outperforms that with quotient features (57.3 vs. 55.5). We argue that it is because that the regular features encode $24$ rotation symmetries using $24$ channels, while $\mathcal{V}$-quotient features only use 6, which may result in a low representation capacity.

\subsection{ISBI 2012 Challenge}
The full training image is $512\times 512\times 30$ voxels in shape. The setting is the same for test images. We take the first $25$ slices of the training image as our training set, and the last 5 as the validation set. For the training set, we extract random $128\times 128\times 5$ voxel patches from the training volume, and reflection pad 48 voxels in the $x$-$y$ plane. Following \cite{quan2016fusionnet}, we apply a random elastic distortion and rotation in the $x$-$y$ plane and add Gaussian noise ($\sigma=0.1$) for data augmentation. Finally, following CubeNet \cite{worrall2018cubenet}, we pass these patches through an equivariant FusionNet \cite{quan2016fusionnet}. We choose hyperparameters with the highest validation score during training. Our models are trained using the Adam optimizer \cite{kingma2014adam}. We train for $4,000$ iterations with a batch size of $4$. We use an initial learning rate of $0.001$ and an exponential decay of $0.99$.

\begin{table}[t]
	\caption{The architecture of the $\mathcal{V}$-steerable model with regular features.}
	\centering
	\begin{tabular}{lccc}
		\toprule
		&  Scalar & Regular & Size \\
		\midrule
		Input  & 1 & 0& $64^3$\\
		Layer1 & 0& 8 & $32^3$ \\
		Layer2 & 0& 8 & $32^3$ \\
		Layer3 & 0& 12 & $16^3$ \\
		Layer4 & 0& 12 & $16^3$ \\
		Layer5 & 0& 16  & $8^3$ \\
		Layer6 & 0& 16 & $8^3$ \\
		Layer7 & 0& 16 & $8^3$ \\
		Layer8 & 512 &0   & $8^3$ \\
		Output & 55 & 0 & $1$ \\
		\bottomrule
	\end{tabular}
	\label{V_shrec}
\end{table}

\begin{table}[t]
	\caption{The architecture of the $\mathcal{T}$-steerable model with regular features.}
	\centering
	\begin{tabular}{lccc}
		\toprule
		&  Scalar & Regular & Size \\
		\midrule
		Input  & 1 & 0& $64^3$\\
		Layer1 & 0& 4 & $32^3$ \\
		Layer2 & 0& 4 & $32^3$ \\
		Layer3 & 0& 8 & $16^3$ \\
		Layer4 & 0& 8 & $16^3$ \\
		Layer5 & 0& 12  & $8^3$ \\
		Layer6 & 0& 12 & $8^3$ \\
		Layer7 & 0& 12 & $8^3$ \\
		Layer8 & 512 &0   & $8^3$ \\
		Output & 55 & 0 & $1$ \\
		\bottomrule
	\end{tabular}
	\label{A4_shrec}
\end{table}

\begin{table}[t]
	\caption{The architecture of the $\mathcal{O}$-steerable model with regular features.}
	\centering
	\begin{tabular}{lccc}
		\toprule
		&  Scalar & Regular & Size \\
		\midrule
		Input  & 1 & 0& $64^3$\\
		Layer1 & 0& 4 & $32^3$ \\
		Layer2 & 0& 4 & $32^3$ \\
		Layer3 & 0& 6 & $16^3$ \\
		Layer4 & 0& 6 & $16^3$ \\
		Layer5 & 0& 8  & $8^3$ \\
		Layer6 & 0& 10 & $8^3$ \\
		Layer7 & 0& 10 & $8^3$ \\
		Layer8 & 512 &0   & $8^3$ \\
		Output & 55 & 0 & $1$ \\
		\bottomrule
	\end{tabular}
	\label{S4_shrec}
\end{table}

\begin{table}[t]
	\caption{The architecture of the $\mathcal{I}$-steerable model with regular features.}
	\centering
	\begin{tabular}{lccc}
		\toprule
		&  Scalar & Regular & Size \\
		\midrule
		Input  & 1 & 0& $64^3$\\
		Layer1 & 0& 3 & $32^3$ \\
		Layer2 & 0& 3 & $32^3$ \\
		Layer3 & 0& 5 & $16^3$ \\
		Layer4 & 0& 5 & $16^3$ \\
		Layer5 & 0& 6  & $8^3$ \\
		Layer6 & 0& 6 & $8^3$ \\
		Layer7 & 0& 6 & $8^3$ \\
		Layer8 & 512 &0   & $8^3$ \\
		Output & 55 & 0 & $1$ \\
		\bottomrule
	\end{tabular}
	\label{A5_shrec}
\end{table}

\begin{table}[t]
	\caption{The architecture of the $SO(3)$-steerable model with irreducible features.}
	\centering
	\begin{tabular}{lcccc}
		\toprule
		&  $D^0(\rho)$ & $D^1(\rho)$  & $D^2(\rho)$ & Size \\
		\midrule
		Input  & 1 &0  & 0& $64^3$\\
		Layer1 & 8 & 4 & 2 & $32^3$ \\
		Layer2 & 8 & 4 & 2 & $32^3$ \\
		Layer3 & 16 & 8 & 4 & $16^3$ \\
		Layer4 & 16 & 8 & 4 & $16^3$ \\
		Layer5 & 32 & 16 & 8 & $8^3$ \\
		Layer6 & 32 & 16 & 8 & $8^3$ \\
		Layer7 & 64 & 32 & 16 & $8^3$ \\
		Layer8 & 512 &0 &0  & $8^3$ \\
		Output & 55 & 0 & 0&$1$ \\
		\bottomrule
	\end{tabular}
	\label{SO3_shrec}
\end{table}

\begin{table}[t]
	\caption{The architecture of the $\mathcal{O}$-steerable model with $\mathcal{V}$-quotient features.}
	\centering
	\begin{tabular}{lccc}
		\toprule
		&  Scalar & $\mathcal{V}$-quotient & Size \\
		\midrule
		Input  & 1 & 0& $64^3$\\
		Layer1 & 0& 8 & $32^3$ \\
		Layer2 & 0& 8 & $32^3$ \\
		Layer3 & 0& 16 & $16^3$ \\
		Layer4 & 0& 16 & $16^3$ \\
		Layer5 & 0& 28  & $8^3$ \\
		Layer6 & 0& 28 & $8^3$ \\
		Layer7 & 0& 30 & $8^3$ \\
		Layer8 & 512 &0   & $8^3$ \\
		Output & 55 & 0 & $1$ \\
		\bottomrule
	\end{tabular}
	\label{S4_quotient_K4_shrec}
\end{table}

\begin{table}[t]
	\caption{The architecture of the $\mathcal{O}$-steerable model with the features composed of regular and $\mathcal{V}$-quotient features.}
	\centering
	\begin{tabular}{lccc}
		\toprule
		&  Scalar & $\mathcal{V}$-quotient & Size \\
		\midrule
		Input  & 1 & 0& $64^3$\\
		Layer1 & 0& 8 & $32^3$ \\
		Layer2 & 0& 8 & $32^3$ \\
		Layer3 & 0& 16 & $16^3$ \\
		Layer4 & 0& 16 & $16^3$ \\
		Layer5 & 0& 28  & $8^3$ \\
		Layer6 & 0& 28 & $8^3$ \\
		Layer7 & 0& 30 & $8^3$ \\
		Layer8 & 512 &0   & $8^3$ \\
		Output & 55 & 0 & $1$ \\
		\bottomrule
	\end{tabular}
	\label{S4_quotient_K4_shrec}
\end{table}

\begin{table}[t]
	\caption{The architecture of the $\mathcal{O}$-steerable model using regular features without the first-order PDOs.}
	\centering
	\begin{tabular}{lccc}
		\toprule
		&  Scalar & Regular & Size \\
		\midrule
		Input  & 1 & 0& $64^3$\\
		Layer1 & 0& 4 & $32^3$ \\
		Layer2 & 0& 6 & $32^3$ \\
		Layer3 & 0& 8 & $16^3$ \\
		Layer4 & 0& 8 & $16^3$ \\
		Layer5 & 0& 10  & $8^3$ \\
		Layer6 & 0& 12 & $8^3$ \\
		Layer7 & 0& 12 & $8^3$ \\
		Layer8 & 512 &0   & $8^3$ \\
		Output & 55& 0 & $1$ \\
		\bottomrule
	\end{tabular}
	\label{S4_no_p1_shrec}
\end{table}

\begin{table}[t]
	\caption{The architecture of the $\mathcal{O}$-steerable model using regular features without the second-order PDOs.}
	\centering
	\begin{tabular}{lccc}
		\toprule
		&  Scalar & Regular & Size \\
		\midrule
		Input  & 1 & 0& $64^3$\\
		Layer1 & 0& 6 & $32^3$ \\
		Layer2 & 0& 6 & $32^3$ \\
		Layer3 & 0& 10 & $16^3$ \\
		Layer4 & 0& 10 & $16^3$ \\
		Layer5 & 0& 16  & $8^3$ \\
		Layer6 & 0& 16 & $8^3$ \\
		Layer7 & 0& 18 & $8^3$ \\
		Layer8 & 512 &0   & $8^3$ \\
		Output & 55& 0 & $1$ \\
		\bottomrule
	\end{tabular}
	\label{S4_no_p2_shrec}
\end{table}

%

\end{document}